\documentclass[letterpaper]{article} % DO NOT CHANGE THIS
\usepackage{aaai2026}  % DO NOT CHANGE THIS
\usepackage{times}  % DO NOT CHANGE THIS
\usepackage{helvet}  % DO NOT CHANGE THIS
\usepackage{courier}  % DO NOT CHANGE THIS
\usepackage[hyphens]{url}  % DO NOT CHANGE THIS
\usepackage{graphicx} % DO NOT CHANGE THIS
\usepackage{natbib}  % DO NOT CHANGE THIS AND DO NOT ADD ANY OPTIONS TO IT
\usepackage{caption} % DO NOT CHANGE THIS AND DO NOT ADD ANY OPTIONS TO IT
\usepackage{color}
\usepackage{multirow}
\usepackage{amsmath}    % align environment
\usepackage{amsthm}     % proof environment
\usepackage{amssymb}    % \leqslant
\usepackage{booktabs}   % \toprule, \midrule, \bottomrule
\usepackage{subcaption}  % 添加 subcaption 包来支持子图
\usepackage{makecell}

\newtheorem{theorem}{Theorem}
\graphicspath{{figure/}}

\usepackage{algorithm}
\usepackage{algorithmic}

\usepackage{newfloat}
\usepackage{listings}

\DeclareCaptionStyle{ruled}{labelfont=normalfont,labelsep=colon,strut=off} % DO NOT CHANGE THIS
\floatstyle{ruled}
\newfloat{listing}{tb}{lst}{}
\floatname{listing}{Listing}
\title{MPD-SGR: Robust Spiking Neural Networks with Membrane Potential Distribution-Driven Surrogate Gradient Regularization}
\author{
    Runhao Jiang\textsuperscript{\rm 1, \rm 2},
    Chengzhi Jiang\textsuperscript{\rm 1, \rm 2},
    Rui Yan\textsuperscript{\rm 3},
    Huajin Tang\textsuperscript{\rm 1, \rm 2, \rm 4}\thanks{Author for correspondence}
}

\affiliations{
    \textsuperscript{\rm 1}College of Computer Science and Technology, Zhejiang University\\
    \textsuperscript{\rm 2}The State Key Lab of Brain-Machine Intelligence, Zhejiang University\\
    \textsuperscript{\rm 3}College of Computer Science and Technology, Zhejiang University of Technology \\
    \textsuperscript{\rm 4}MOE Frontier Science Center for Brain Science and Brain-Machine Integration, Zhejiang University \\
    RhJiang@zju.edu.cn, jiangchengzhi2024@zju.edu.cn, Ryan@zjut.edu.cn, htang@zju.edu.cn
}

\usepackage{bibentry}
\begin{document}

\maketitle

\begin{abstract}
    The surrogate gradient (SG) method has shown significant promise in enhancing the performance of deep spiking neural networks (SNNs), but it also introduces vulnerabilities to adversarial attacks. Although spike coding strategies and neural dynamics parameters have been extensively studied for their impact on robustness, the critical role of gradient magnitude, which reflects the model's sensitivity to input perturbations, remains underexplored. In SNNs, the gradient magnitude is primarily determined by the interaction between the membrane potential distribution (MPD) and the SG function. In this study, we investigate the relationship between the MPD and SG and their implications for improving the robustness of SNNs. Our theoretical analysis reveals that reducing the proportion of membrane potentials lying within the gradient-available range of the SG function effectively mitigates the sensitivity of SNNs to input perturbations. Building upon this insight, we propose a novel MPD-driven surrogate gradient regularization (MPD-SGR) method, which enhances robustness by explicitly regularizing the MPD based on its interaction with the SG function. Extensive experiments across multiple image classification benchmarks and diverse network architectures confirm that the MPD-SGR method significantly enhances the resilience of SNNs to adversarial perturbations and exhibits strong generalizability across diverse network configurations, SG functions, and spike encoding schemes.
\end{abstract}

%     % \link{Code}{https://aaai.org/example/code}
%     % \link{Datasets}{https://aaai.org/example/datasets}

\section{Introduction}
Robustness is an intrinsic characteristic of the human brain, a trait that traditional artificial neural networks (ANNs) fail to replicate. Unlike ANNs that rely on dense embeddings~\cite{wang2022learngene,wang2024vision,anonymous2024clusterlearngene}, Spiking neural networks (SNNs)~\cite{maass1997networks} emulate the structural and functional properties of the brain by biological neurons that encode information via binary spikes, thereby demonstrating enhanced resilience to noise and perturbations~\cite{marchisio2020spiking,sharmin2019comprehensive,hu2024toward}. Although SNNs currently underperform ANNs in terms of accuracy, their intrinsic robustness advantage makes them particularly suitable for high-security applications like autonomous driving~\cite{zhu2024autonomous} and privacy computing~\cite{Kim_Venkatesha_Panda_2022}.

The development of surrogate gradient (SG) methods has significantly advanced the performance of deep SNNs~\cite{tavanaei2019deep}, yet this progress introduces vulnerabilities to gradient-based adversarial attacks~\cite{liang2021exploring}. Recent studies have shown that, although SNNs inherently exhibit greater robustness than ANNs, such adversarial perturbations can critically compromise their accuracy~\cite{ding2022snn}. Enhancing the adversarial robustness of SNNs remains a pivotal challenge in the research community. Some techniques originally developed for ANNs, such as adversarial training~\cite{dingrobust} and Lipschitz regularization~\cite{ding2022snn}, have been extended to SNNs. In parallel, growing attention has been paid to understanding how structural parameters (e.g., leak factors~\cite{chowdhury2021towards,sharmin2020inherent,xu2024feel}, thresholds~\cite{el2021securing}) and neural encoding schemes~\cite{ma2023exploiting,ding2024enhancing,wu2024rsc} affect the robustness of SNNs. The intrinsic noise-filtering properties of spiking dynamics, along with the stochastic nature of neural coding, are considered critical factors underpinning the robustness of SNNs.

However, existing research~\cite{xu2024feel,liu2024enhancing} has demonstrated that a model's sensitivity to input perturbations (i.e., robustness error) can be reflected in the magnitude of gradients, which is primarily governed by the interaction between the membrane potential distribution (MPD) and the SG function in SNNs. The MPD not only reflects the response characteristics of the neuronal population but also significantly influences the gradient propagation across network layers. Recent studies~\cite{guo2022reducing,guo2022recdis,guo2023rmp} have further shown that constraining the MPD to remain within the gradient-available region of the SG function with an appropriate proportion facilitates more effective gradient-based optimization. Therefore, it is also crucial to investigate how to constrain the MPD according to the SG function to enhance the robustness of SNNs.

In this paper, we present a theoretical and empirical investigation into the role of neuronal MPD and SG mechanisms in enhancing the adversarial robustness of SNNs. First, through rigorous gradient analysis, we establish a theoretical framework that formally connects robustness error with the SG in SNNs, demonstrating that reducing the SG magnitude can effectively mitigate robustness error. Since the SG magnitude is primarily influenced by the membrane potential, we further investigate the MPD and establish an explicit connection between the MPD and network parameters. Thus, the MPD can be optimized by learning the network parameters. Lastly, based on the relationship between the MPD and the SG function, we derive an approximate expression for the SG magnitude. Based on these theoretical insights, we design an MPD-driven surrogate gradient regularization (MPD-SGR) method that reduces the proportion of MPD within the gradient-available region of the SG function, thereby enhancing the network's resilience against adversarial attacks (Figure~\ref{fg:mpd-sgr}). The proposed framework demonstrates remarkable adaptability across diverse neuron parameter configurations, SG function variants, and spike coding methods. Experimental results across multiple image classification datasets, network architectures, and gradient approximation methods demonstrate that our MPD-SGR method effectively improves the robustness of SNNs against various input perturbations.

\section{Related Work}
\subsection{Surrogate Gradient Learning}
Surrogate gradient (SG) methods provide approximate gradients for the non-differentiable spiking neurons, thus supporting the backpropagation of errors in both spatial and temporal domains~\cite{wu2018spatio,gu2019stca,neftci2019surrogate}. Recent studies have found that the relationship between the gradient-available interval of the SG function and the membrane potential dynamics has a significant impact on the backpropagation training of SNNs~\cite{zenke2021remarkable}, inspiring a series of works to study the alignment between membrane potential and gradient-available interval. Some researchers adopt a fixed SG function and constrain the membrane potential distribution to ensure that the gradient is propagated in an appropriate proportion. InfLoR-SNN~\cite{guo2022reducing} designed a membrane potential rectifier to redistribute the membrane potential closer to the spiking threshold. RecDis-SNN~\cite{guo2022recdis} introduced three regularization losses to penalize three undesired shifts of the membrane potential distribution. Optimizing the SG function is another appealing approach. Dspike~\cite{li2021differentiable} adaptively changed its shape and captured the direction of finite difference gradients to find the optimal shape and smoothness for gradient estimation. LSG~\cite{lian2023learnable} adapted the width of the SG function automatically during training based on the membrane potential dynamics. ASGL~\cite{wang2023adaptive} learned the precise gradients of the loss landscape in SNNs adaptively by fusing the learnable relaxation degree into a prototype network with random spike noise.

\subsection{Adversarial Defense for SNN}
SNNs exhibit inherent robustness~\cite{marchisio2020spiking,sharmin2019comprehensive}, a property that is often attributed to their biologically plausible components, such as neural coding and dynamics. Several studies aim to enhance robustness from a neurodynamic perspective. For example, structural parameters, including time windows, leak factors, and thresholds, play a key role in determining the adversarial robustness of SNNs~\cite{el2021securing}. The noise-filtering effect of membrane potential leakage characteristics has motivated approaches to improve the SNN robustness by evolving leaky factor~\cite{chowdhury2021towards,sharmin2020inherent}. Additionally, leveraging stochasticity or attention mechanisms in neural coding can mitigate noise in the information transmission process. For instance, Poisson coding, which encodes the continuous intensity of an image into binary spikes with inherent randomness, has been shown to exhibit greater robustness than direct coding~\cite{sharmin2020inherent,kim2022rate}. To further utilize the noisy, non-deterministic nature of neural coding, NDL~\cite{ma2023exploiting} incorporates noisy neuronal dynamics into SNNs to investigate their potential robustness benefits, while StoG~\cite{ding2024enhancing} introduces additional stochastic gating for spiking neurons to enhance model robustness. The FEEL method~\cite{xu2024feel} adopts a frequency encoding inspired by selective visual attention mechanisms, suppressing noise in different frequency ranges at different time steps. Moreover, there are also some SNN-specific defensive techniques that draw inspiration from ANNs. A typical representative is adversarial training (AT)~\cite{kundu2021hire}, which enhances robustness by incorporating adversarial examples generated through attacks into the training process. RAT~\cite{ding2022snn} proposes a regularized training strategy based on Lipschitz analysis of SNNs.

\begin{figure*}[t]
    \centering
    \includegraphics[width=0.7\textwidth, height=0.4\textwidth]{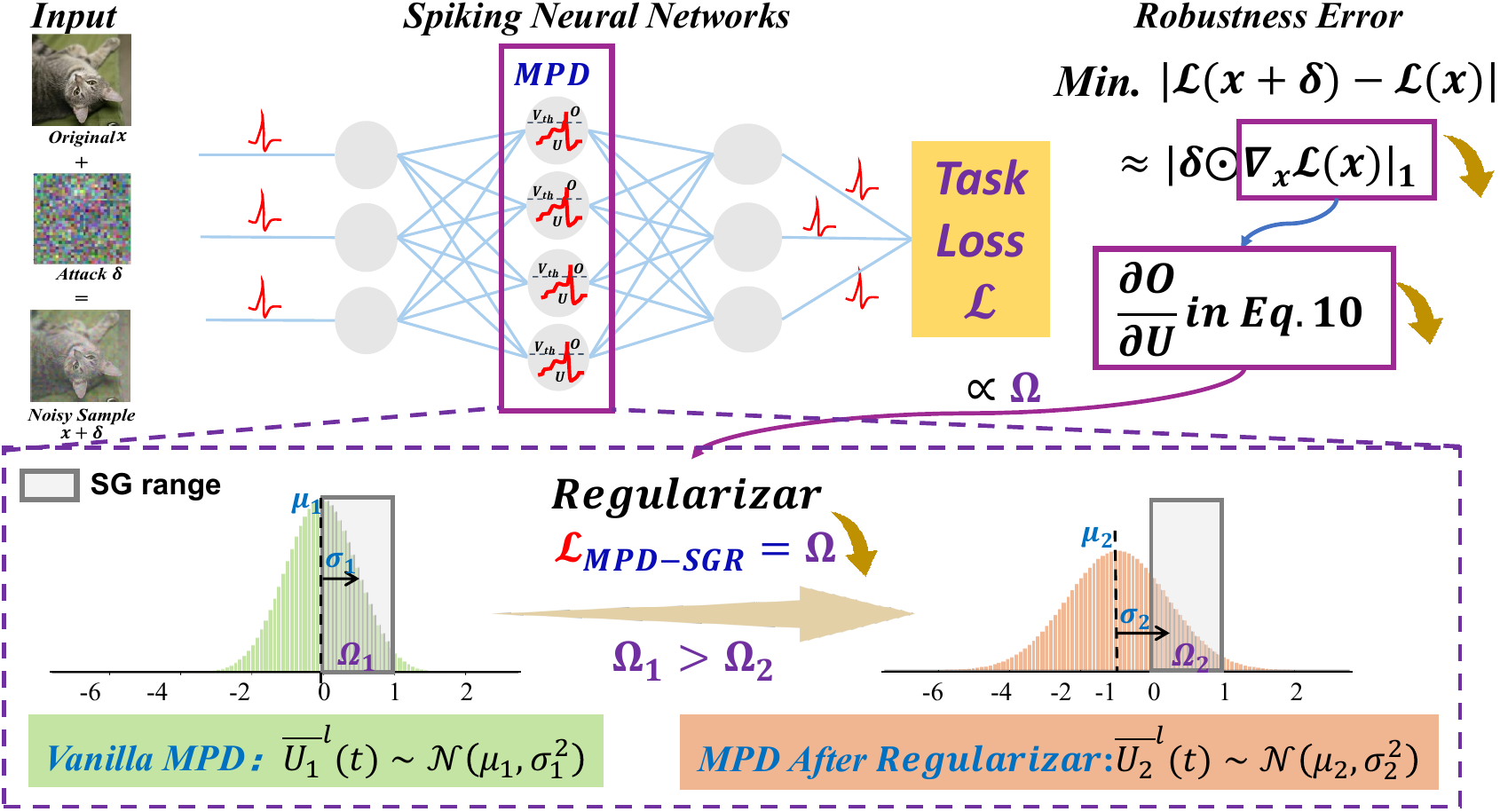}
    \caption{The overall framework of MPD-SGR. The MPD-SGR constrains membrane potential distribution (MPD) of SNNs (mean $\mu$ and standard deviation $\sigma$) to reduce the overlap $\Omega$ between MPD and the gradient-available range of the SG function (gray area). This regularization minimizes output error under adversarial perturbations.}
    \label{fg:mpd-sgr}
\end{figure*}

\section{Preliminaries}
\subsection{Spiking Neural Networks}
In the field of deep SNNs, the most commonly used spiking neuron is the iterative leaky integrate-and-fire (LIF) model. The iterative LIF model simulates biological neurons through three key dynamic processes: synaptic integration, membrane potential accumulation, and neuronal firing with resetting, which are mathematically defined as follows:
\begin{align}
     & I^{l-1}(t) = W^l \cdot O^{l-1}(t),                                    \\
     & R^l(t-1) = O^{l}(t-1) r^{l}(t-1),                                     \\
     & U^{l}(t) = \tau (U^l(t-1)- R^l(t-1)) +I^{l-1}(t), \label{eq:membrane} \\
     & O^{l}(t) = \mathcal{H}(\overline{U}^{l}(t)-v_{th}), \label{eq:firing}
\end{align}
where $l$ and $t$ denote the layer index and the time step of neural activity, respectively. The postsynaptic current $I^{l-1}(t)$ is the sum of spike signals $O^{l-1}(t)$ weighted by the synaptic weights $W^l$. $R^l(t-1)$ is the reset term, where the triggering spike $O^{l}(t-1)$ causes the membrane potential reset and $r^{l}(t)$ is the subtracted value due to reset. Under the hard reset condition, $r^{l}(t) = U^{l}(t)$, whereas under the soft reset condition, $r^{l}(t) = v_{th}$. $U^{l}(t)$ denotes the membrane potential state, receiving the postsynaptic current and decaying with a leaky factor $\tau$. When $U^{l}(t)$ exceeds the threshold $v_{th}$, a spike is generated according to the Heaviside function $\mathcal{H}(\cdot)$.

\subsection{Adversarial Attacks for SNN}
Neural networks are notorious for their vulnerability to subtle perturbations in the input data, known as adversarial attacks. Adversarial attacks generate a perturbation $\delta$ by maximizing the network's task loss $\mathcal{L}$ and then applying it to input data $x$, resulting in adversarial examples. Formally, this maximization problem can be expressed as:
\begin{align}
    \delta = \arg\max \mathcal{L}(f(x+\delta; W), y)  \quad  s.t. \delta \in R(x, \epsilon)
\end{align}
where $y$ is the target label, $f$ is the network parameterized by its weights $W$, and $R(x, \epsilon)$ is the $\ell_p$-constrained neighborhood centered on $x$ with radius $\epsilon$, ensuring the perturbation remains imperceptible.

Researchers have observed that adversarial attacks can also be applied to deep SNNs using SG~\cite{bu2023rate}, where the SG function provides approximate gradients for non-differentiable Heaviside step function. In the SG method, the backpropagation of the loss is derived by unfolding the dynamics of LIF neurons as follows:
\begin{align}
    \frac{\partial\mathcal{L}}{\partial\boldsymbol{U}^l(t)}=\frac{\partial\mathcal{L}}{\partial\boldsymbol{O}^{l}(t)}\frac{\partial\boldsymbol{O}^{l}(t)}{\partial\boldsymbol{U}^{l}(t)}
    +\frac{\partial\mathcal{L}}{\partial\boldsymbol{U}^l(t+1)}\frac{\partial\boldsymbol{U}^l(t+1)}{\partial\boldsymbol{U}^l(t)}
\end{align}
where $\mathcal{L}$ is the task loss. SG methods approximate the non-differentiable term $\frac{\partial\boldsymbol{O}^{l}(t)}{\partial\boldsymbol{U}^{l}(t)}$ with an SG function. Among the various SG functions demonstrated to be effective~\cite{neftci2019surrogate}, we employ the triangle SG function defined as:
\begin{equation} \label{eq:sg}
    \frac{\partial\boldsymbol{O}^{l}(t)}{\partial\boldsymbol{U}^{l}(t)} \approx h\left(\boldsymbol{\overline{U}}^{l}(t)\right) = \frac{1}{\gamma^2}\max\left(\gamma-\left|\boldsymbol{\overline{U}}^{l}(t)\right|,0\right)
\end{equation}
where $\boldsymbol{\overline{U}}^{l}(t) = \boldsymbol{U}^{l}(t) - v_{th}$ and $\gamma$ is a hyperparameter that controls the smoothness and gradient-available interval of the surrogate function.

\section{Methods}
\subsection{Analysis of robustness error}
The effect of a perturbation is quantified as the difference in error values before and after perturbation (i.e., robustness error), $\mathcal{L}(x + \delta) - \mathcal{L}(x)$. This perturbation-induced error difference is theoretically analyzed using local linearity techniques~\cite{qin2019adversarial}, yielding the following expression:
\begin{eqnarray}
    \left|\mathcal{L}\left(x+\delta\right)-\mathcal{L}\left(x\right)\right|\leqslant
    \left|\delta\odot\nabla_{x}\mathcal{L}\left(x\right)\right|_{1}+g\left(\delta,x\right)
\end{eqnarray}
where $g\left(\delta,x\right)$ is the residual term and $\nabla_{x}\mathcal{L}\left(x\right)$ is the input gradient. This theorem reveals a relationship between variations in robustness error and the $L_1$ norm of the input gradient. A natural approach to reduce these variations is to regularize the input gradient. However, this requires double backpropagation to compute its gradient with respect to model parameters, which is computationally expensive and impractical for large-scale SNNs.

To address this, we leverage LIF dynamics and BPTT to reformulate input-gradient optimization as internal network-gradient optimization~\cite{xu2024feel}. Specifically, the constraint term for SNNs can be rewritten as $\sum_{t} |\delta_t\odot\nabla_{O}\mathcal{L}(O_t)|_{1}$, which can be further derived as follows:
\begin{equation} \label{eq:constraint}
    \min \sum_t \left| \delta_t \odot \frac{\partial \mathcal{L}}{\partial O_t} \right|_1 = \min \sum_t \left| \frac{1}{L} \sum_{l=1}^L \left( P_1 \cdot P_2 \cdot P_3 \right) \frac{\partial \mathcal{L}}{O_l^T}  \right|_1
\end{equation}

Equation~\ref{eq:constraint} outlines the key components influencing robustness, including the perturbation term $P_1 = \prod_{k=t}^T \delta(t) \odot \tau_l^k$, the model weight term $P_2 = \prod_{q=2}^l W_{q-1, q}$ and the SG term $P_3 = \prod_{v=1}^l \frac{\partial O_v^t}{\partial U_v^t}$. These constraints offer a principled way to regularize SNNs for robustness, explaining why recent methods such as weight regularization~\cite{ding2022snn} and evolving leak factors~\cite{xu2024feel} are effective. However, the potential of SG terms for improving robustness has been largely overlooked, which is the motivation of this work. Beyond the gradient-based analysis, we also derive an upper bound on robustness sensitivity from the perspective of neuron firing states (Appendix~A).

\subsection{Analysis of Membrane Potential Distribution}
Equation~\ref{eq:constraint} shows that regularizing the SG terms suppresses fluctuations in robustness error. As defined in Equation~\ref{eq:sg}, the SG magnitude depends on the SG function and the membrane potential ($\overline{U}$). Therefore, analyzing the membrane potential distribution (MPD) is a prerequisite for understanding how SG affects SNN robustness. In forward propagation, the MPD is constrained by the threshold-dependent batch normalization (tdBN)~\cite{zheng2021going} and the dynamics of LIF neurons. tdBN first transforms the postsynaptic current $I$ into a Gaussian-distributed variable $\overline{I}$ using the scaling factors ($\alpha$ and $V_{th}$) and a learnable affine transformation ($\lambda$ and $\beta$). The dynamics of LIF neurons reveal that the MPD is governed by the decay factor ($\tau$), leading to a shift and scaling relative to the normalized postsynaptic current $\overline{I}$. Based on these observations, we propose Theorem 1 to explain the detailed dynamic distribution of membrane potential.

\begin{theorem}
    In an iterative LIF model with decay factor $\tau$ over $T$ timesteps, the tdBN-normalized postsynaptic input follows the distribution $\overline{I} \sim \mathcal{N}(\beta_{c}, (\lambda_{c}\alpha V_{th})^2)$. For $t=1, 2,3,...,T$, the membrane potentials is distributed as $\overline{U}_{c}^{l}(t) \sim \mathcal{N}(\beta_c D(\tau, t) - S(t) , (\lambda_{c}\alpha V_{th})^2 D(\tau^2, t))$, where $c$ is the channel number of the tdBN layer, $D(\tau, t) = \sum_{i=1}^{t} \tau^{t-i}$ is the cumulative decay function, and $S(t)$ is the cumulative response strength constant.
\end{theorem}

\begin{proof}
    The proof of Theorem 1 is included in Appendix~A.
\end{proof}

\begin{table*}[t]
    \centering
    \begin{tabular}{ccccccccc}
        \hline
        \multicolumn{1}{c|}{Methods}                       & \multicolumn{4}{c|}{CIFAR-10} & \multicolumn{4}{c}{CIFAR-100}                                                                                                                            \\
        \multicolumn{1}{c|}{}                              & Clean                         & FGSM                          & PGD            & \multicolumn{1}{c|}{BIM}            & Clean          & FGSM           & PGD            & BIM            \\ \hline
        \multicolumn{9}{c}{\textbf{Vanilla Training}}                                                                                                                                                                                                 \\ \hline
        \multicolumn{1}{c|}{REG~\cite{ding2022snn}}        & \textbf{92.49}                & 25.18                         & 0.88           & \multicolumn{1}{c|}{0.60}           & \textbf{72.82} & 10.14          & 0.27           & 0.31           \\
        \multicolumn{1}{c|}{StoG~\cite{ding2024enhancing}} & 91.64                         & 16.22                         & 0.28           & \multicolumn{1}{c|}{0.12}           & 72.22          & 5.92           & 0.26           & 0.20           \\
        \multicolumn{1}{c|}{DLIF~\cite{dingrobust}}        & 92.01                         & 11.52                         & 0.08           & \multicolumn{1}{c|}{0.06}           & 71.38          & 7.20           & 0.08           & 0.08           \\
        \multicolumn{1}{c|}{FEEL~\cite{xu2024feel}}        & 90.08                         & 29.17                         & 6.67           & \multicolumn{1}{c|}{5.99}           & 70.06          & 9.74           & 2.06           & 1.92           \\
        \multicolumn{1}{c|}{SR~\cite{liu2024enhancing}}    & 91.04                         & 31.72                         & 8.55           & \multicolumn{1}{c|}{7.28}           & 66.76          & 16.16          & 8.00           & 6.43           \\
        \hline
        \multicolumn{1}{c|}{Ours}                          & 91.63                         & \textbf{47.59}                & \textbf{20.55} & \multicolumn{1}{c|}{\textbf{16.85}} & 70.42          & \textbf{34.51} & \textbf{9.03}  & \textbf{8.41}  \\
        \multicolumn{1}{c|}{Improvement}                   & -0.86                         & +15.87                        & +12            & \multicolumn{1}{c|}{+9.57}          & -2.4           & +18.35         & +1.03          & +1.98          \\ \hline
        \multicolumn{9}{c}{\textbf{Adversarial Training}}                                                                                                                                                                                             \\ \hline
        \multicolumn{1}{c|}{RAT~\cite{ding2022snn}}        & \textbf{91.41}                & 45.00                         & 22.95          & \multicolumn{1}{c|}{20.80}          & 69.43          & 19.07          & 9.23           & 8.41           \\
        \multicolumn{1}{c|}{StoG~\cite{ding2024enhancing}} & 90.13                         & 45.75                         & 27.74          & \multicolumn{1}{c|}{26.32}          & 69.24          & 19.64          & 9.77           & 3.23           \\
        \multicolumn{1}{c|}{DLIF~\cite{dingrobust}}        & 88.94                         & 39.21                         & 27.17          & \multicolumn{1}{c|}{25.98}          & 67.08          & 19.34          & 9.96           & 9.39           \\
        \multicolumn{1}{c|}{FEEL~\cite{xu2024feel}}        & 89.00                         & 45.62                         & 29.52          & \multicolumn{1}{c|}{28.39}          & 68.05          & 19.55          & 12.11          & 11.97          \\
        \multicolumn{1}{c|}{SR~\cite{liu2024enhancing}}    & 88.26                         & 44.28                         & 28.63          & \multicolumn{1}{c|}{27.03}          & 61.26          & 23.10          & 17.07          & 16.28          \\
        \hline
        \multicolumn{1}{c|}{Ours}                          & 90.69                         & \textbf{59.27}                & \textbf{33.38} & \multicolumn{1}{c|}{\textbf{32.61}} & \textbf{69.56} & \textbf{39.45} & \textbf{22.23} & \textbf{19.45} \\
        \multicolumn{1}{c|}{Improvement}                   & -0.72                         & +13.52                        & +3.86          & \multicolumn{1}{c|}{+4.22}          & +0.13          & +16.35         & +5.16          & +3.17          \\ \hline
    \end{tabular}
    \caption{Compare with state-of-the-art work on adversarial robustness of SNN (VGG11, $T$ = 8)} \label{tb:sota}
\end{table*}

\begin{table*}[t]
    \setlength{\tabcolsep}{1.5mm}
    \begin{center}
        \begin{sc}
            \begin{tabular}{ccccccc}
                \toprule
                \textbf{Dataset}          & \textbf{Model}  & \textbf{Clean}                        & \textbf{FGSM}        & \textbf{PGD}         & \textbf{BIM}         & \textbf{CW}          \\
                \midrule
                \multirow{8}{*}{CIFAR10}  & VGG11, BPTT     & \multirow{2}{*}{\textbf{92.45}/90.62} & 23.86/\textbf{42.12} & 0.87/\textbf{17.18}  & 0.60/\textbf{14.59}  & 6.26/\textbf{18.85}  \\
                                          & VGG11, BPTR     &                                       & 27.92/\textbf{47.61} & 7.82/\textbf{25.52}  & 6.61/\textbf{25.00}  & 24.38/\textbf{38.07} \\
                                          & VGG11, BPTT, AT & \multirow{2}{*}{\textbf{91.08}/90.34} & 43.80/\textbf{58.71} & 20.63/\textbf{28.94} & 18.63/\textbf{24.74} & 29.87/\textbf{34.38} \\
                                          & VGG11, BPTR, AT &                                       & 49.11/\textbf{66.17} & 35.74/\textbf{43.12} & 34.59/\textbf{41.17} & 58.73/\textbf{65.05} \\
                                          & WRN16, BPTT     & \multirow{2}{*}{\textbf{93.46}/92.22} & 18.81/\textbf{45.82} & 0.02/\textbf{8.44}   & 0.03/\textbf{6.59}   & 3.61/\textbf{18.66}  \\
                                          & WRN16, BPTR     &                                       & 15.28/\textbf{38.18} & 0.17/\textbf{8.03}   & 0.16/\textbf{6.65}   & 11.76/\textbf{31.21} \\
                                          & WRN16, BPTT, AT & \multirow{2}{*}{91.15/\textbf{91.34}} & 42.31/\textbf{63.32} & 19.93/\textbf{38.11} & 18.03/\textbf{33.13} & 29.60/\textbf{43.06} \\
                                          & WRN16, BPTR, AT &                                       & 51.75/\textbf{70.93} & 34.44/\textbf{46.91} & 33.00/\textbf{44.79} & 58.90/\textbf{70.10} \\
                \hline
                \multirow{8}{*}{CIFAR100} & VGG11, BPTT     & \multirow{2}{*}{\textbf{72.54}/69.56} & 8.95/\textbf{20.16}  & 0.24/\textbf{6.90}   & 0.20/\textbf{5.51}   & 5.62/\textbf{13.54}  \\
                                          & VGG11, BPTR     &                                       & 11.45/\textbf{34.49} & 4.66/\textbf{10.07}  & 4.44/\textbf{9.94}   & 25.03/\textbf{26.99} \\
                                          & VGG11, BPTT, AT & \multirow{2}{*}{\textbf{68.60}/67.67} & 20.28/\textbf{36.18} & 9.97/\textbf{18.23}  & 8.82/\textbf{15.70}  & 15.46/\textbf{22.33} \\
                                          & VGG11, BPTR, AT &                                       & 29.78/\textbf{40.42} & 22.14/\textbf{26.38} & 21.27/\textbf{25.88} & 40.96/\textbf{47.44} \\

                                          & WRN16, BPTT     & \multirow{2}{*}{\textbf{73.86}/72.61} & 10.99/\textbf{22.58} & 0.04/\textbf{1.97}   & 0.05/\textbf{1.77}   & 4.93/\textbf{10.34}  \\
                                          & WRN16, BPTR     &                                       & 9.74/\textbf{21.60}  & 0.44/\textbf{5.52}   & 0.35/\textbf{5.50}   & 14.81/\textbf{24.13} \\
                                          & WRN16, BPTT, AT & \multirow{2}{*}{69.29/\textbf{69.67}} & 25.35/\textbf{39.47} & 12.09/\textbf{21.71} & 11.30/\textbf{18.87} & 20.17/\textbf{27.55} \\
                                          & WRN16, BPTR, AT &                                       & 31.99/\textbf{44.23} & 22.58/\textbf{35.67} & 22.04/\textbf{35.08} & 43.23/\textbf{53.21} \\
                \bottomrule
            \end{tabular}
        \end{sc}
    \end{center}
    \caption{The classification accuracy (Vanilla/Ours) under white-box attacks across multiple datasets and architectures ($T=4$).} \label{tb:wb}
\end{table*}

\subsection{The relationship between MPD and SG}
In SG methods, the SG function determines the gradient-available interval of the membrane potential and provides an approximate gradient. The membrane potential distribution is modeled as a Gaussian distribution according to Theorem 1, where gradient information is primarily carried by the membrane potential within the overlap between this distribution and the gradient-available interval of the SG function. A smaller overlap obstructs gradient propagation, while an excessively large overlap introduces numerous inaccurate approximate gradients, increasing the deviation from the true gradients (See Appendix~B for a detailed explanation). To optimize SNNs more effectively, much of the current work focuses on identifying the optimal overlap areas by adjusting the MPD~\cite{guo2022reducing,guo2022recdis} or the gradient-available interval~\cite{lian2023learnable,wang2023adaptive}. Their key objective is to ensure that the proportion of membrane potentials with gradients is appropriate, allowing for stable gradient propagation and reliable training. However, the SG not only plays a pivotal role in network training but also significantly impacts the robustness error induced by perturbations ($P_3$ term in Equation~\ref{eq:constraint}). Therefore, to achieve reliable model output, the SG magnitude needs to be appropriately reduced while avoiding interference with network training, which motivates us to impose thoughtful constraints on the SG magnitude to ensure training and robust performance.

The SG magnitude is determined by the overlap area, so we first derive the expression for the overlap area to represent the SG magnitude. We assume that the SG function satisfies Equation~\ref{eq:sg}, with its gradient-available interval defined as $[-\gamma, \gamma]$. The MPD follows a Gaussian distribution $\overline{U}^{l}(t) \sim \mathcal{N}(\mu, \sigma^2)$, with the probability density function (PDF) given by $p(x) = \frac{1}{\sigma\sqrt{2\pi}}\exp \left(-\frac{(x-\mu)^2}{2\sigma^2}\right)$. The overlap area is defined as the integral of the Gaussian distribution over the interval $[-\gamma, \gamma]$:
\begin{eqnarray}
    \Omega =\int_{-\gamma}^{\gamma}p(x)dx = \frac{1}{\sigma\sqrt{2\pi}}\int_{-\gamma}^{\gamma}\exp \left(-\frac{(x-\mu)^2}{2\sigma^2}\right)dx
\end{eqnarray}

To simplify the integral, we perform a change of variables. Let $z = \frac{x-\mu}{\sigma}$, so that $dx = \sigma dz$. Under this substitution, we obtain:
\begin{eqnarray}
    \Omega =\int_{\frac{\mu-\gamma}{\sigma}}^{\frac{\mu+\gamma}{\sigma}} \frac{1}{\sqrt{2\pi}} \exp \left(-\frac{z^2}{2}\right)dz
\end{eqnarray}

The integral can be evaluated using the cumulative distribution function (CDF)  of the standard normal distribution:
\begin{eqnarray} \label{eqn:cdf}
    \Phi(x) = \frac{1}{2} \left[ 1+ \operatorname{erf}(\frac{x}{\sqrt{2}})\right]
\end{eqnarray}
where $\operatorname{erf}(x) = \frac{2}{\sqrt{\pi}} \int_{0}^{x}  e ^{-t^2}dt$ is the error function. Using this CDF, the target overlap $\Omega$ can be written as:
\begin{eqnarray} \label{eqn:overlap}
    \Omega = \Phi\left(\frac{\mu+\gamma}{\sigma}\right) - \Phi\left(\frac{\mu-\gamma}{\sigma}\right)
\end{eqnarray}

When the SG function parameters $\gamma$ is fixed, the value of $\Omega$ is controlled by the mean $\mu$ and standard deviation $\sigma$ of membrane potential. The gradient of $\Omega$ with respect to the parameters of the membrane potential distribution is computed as follows:
\begin{align}
     & \frac{\partial\Omega}{\partial\mu}=\frac{1}{\sigma}\left[p\left(\frac{\mu^+}{\sigma}\right)-p\left(\frac{\mu^-}{\sigma}\right)\right]                 \\
     & \frac{\partial\Omega}{\partial\sigma}=\frac{1}{\sigma^2}\left[-\mu^+p\left(\frac{\mu^+}{\sigma}\right)+\mu^-p\left(\frac{\mu^-}{\sigma}\right)\right]
\end{align}
where $\mu^+ = \mu+\gamma$ and $\mu^- = \mu-\gamma$ are the upper and lower bounds of the overlap area, respectively.

Therefore, Equation~\ref{eqn:overlap} can serve as a SG regularizer, reducing the SG magnitude by constraining the parameters of the MPD (Figure~\ref{fg:mpd-sgr}). In practice, except for the last linear output layer, the SG is penalized at every timestamp in every channel for each layer in the training phase. Specifically, $\mathcal{L}_{MPD-SGR}$ for the b-th batch of input is given by:
\begin{align}
     & \mathcal{L}_{MPD-SGR}^b = \frac{1}{LCT}\sum_{l,c,t} \Omega_{c}^{b,l}(t)                                                                                                      \\
     & = \frac{1}{LCT}\sum_{l,c,t} \left[\Phi\left(\frac{\mu_{c}^{l}(t)+\gamma}{\sigma_{c}^{l}(t)}\right) - \Phi\left(\frac{\mu_{c}^{l}(t)-\gamma}{\sigma_{c}^{l}(t)}\right)\right]
\end{align}
where $L$ is the number of layers, $C$ is the number of channels, and $T$ is the number of time step. $\Omega_{c}^{l}(t)$ is the overlap area of the membrane potential distribution at the channel $c$ of the layer $l$ at time $t$. Finally, taking classification loss into consideration, the total loss can be written as:
\begin{eqnarray} \label{eqn:loss}
    \mathcal{L}^b = \mathcal{L}_{task}^b + \eta \mathcal{L}_{MPD-SGR}^b
\end{eqnarray}
where $\mathcal{L}_{task}^b$ denotes the task loss (e.g., cross-entropy), and $\eta$ is a coefficient that modulates the strength of the regularization. The training evolution of the MPD and the SG overlap ($\Omega$), as well as an analysis of $\eta$ are provided in Appendix~C.

\begin{figure}[t]
    \centering
    \includegraphics[width=0.47\textwidth]{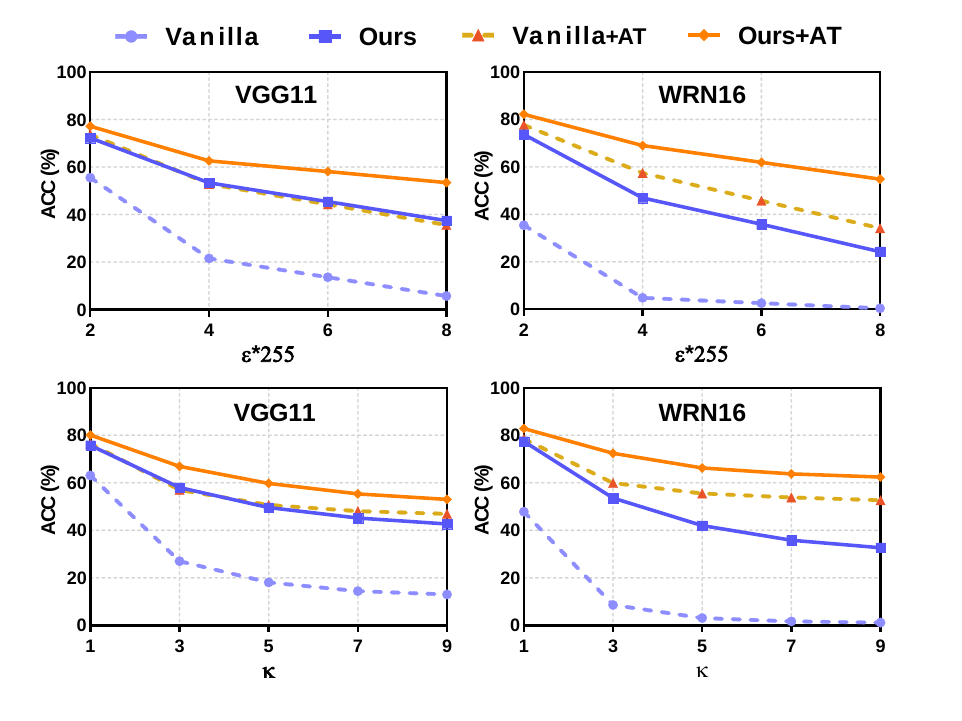}
    \caption{Performance of the white-box PGD attack with increasing perturbation $\epsilon$ and iterative step $k$ = 4 (Top Panels), increasing iterative step $k$ and $\epsilon$ = 8/255 (Bottom Panels).}
    \label{fg:wb-plus}
\end{figure}

\section{Experiments}
\subsection{Experimental settings}
We conduct experiments to evaluate our proposed method on image classification tasks using the CIFAR-10, CIFAR-100 and Tiny-ImageNet~\cite{le2015tiny} datasets. The SNN models are VGG (VGG11)~\cite{simonyan2014very} and WideResNet (WRN16)~\cite{zagoruyko2016wide}. The LIF neuron parameters and the SG function settings are consistent with previous work~\cite{ding2022snn}.

We evaluate model robustness under three attack scenarios: white-box, black-box and non-gradient attacks. The adversarial examples are generated by four attack methods using differentiable approximation techniques (BPTT or BPTR~\cite{bu2023rate}): FGSM~\cite{goodfellow2014explaining}, PGD~\cite{madry2017towards}, BIM~\cite{kurakin2018adversarial} and CW~\cite{carlini2017towards}. The attack perturbation strength is set to $\epsilon= 8/255$, the iterative steps $k= 7$ and step size $\alpha= 0.01$ for PGD and BIM. For adversarial training (AT), models are trained with white-box PGD adversarial examples ($k$ = 2, $\epsilon$ = 2/255). The bold values in all tables represent the optimal results for each setting. Full details on experimental settings and attack algorithms are provided in Appendix~D.

\begin{table}[t]
    \centering
    \begin{tabular}{cccc}
        \hline
        Methods & Clean & Gaussian Noise          & Uniform Noise          \\ \hline
        REG     & 72.82 & 24.73 (-48.09)          & 49.86 (-23.86)         \\
        FEEL    & 70.06 & 32.63 (-37.43)          & 52.47 (-17.59)         \\
        SR      & 66.76 & 47.47 (-19.29)          & 61.37 (\textbf{-5.39}) \\ \hline
        Ours    & 70.42 & \textbf{53.01 (-17.41)} & \textbf{64.21} (-6.21) \\ \hline
    \end{tabular}
    \caption{Comparative evaluation of defense mechanisms under random perturbation attacks on CIFAR100 datasets with VGG11 ($T$ = 8). Values in parentheses denote accuracy degradation relative to clean performance.} \label{tb:atk_nograd}
\end{table}

\begin{figure}[t]
    \centering
    \includegraphics[width=0.47\textwidth]{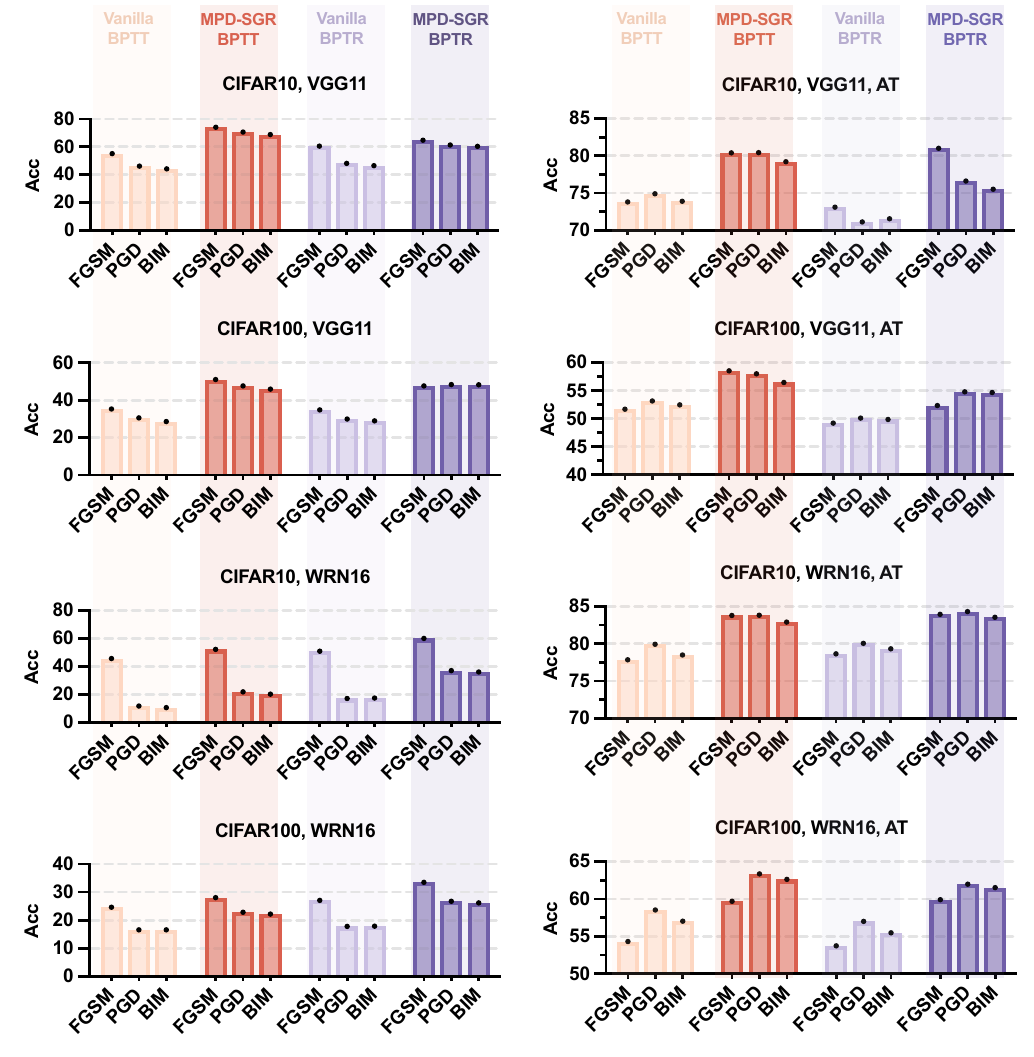}
    \caption{Performance of the proposed MPD-SGR method under different black-box attacks.}
    \label{fg:bb}
\end{figure}

\subsection{Performance for various attack types}
\subsubsection{Comparison with state-of-the-art work.}
We evaluate MPD-SGR against state-of-the-art (SOTA) adversarial defense methods for SNNs including REG~\cite{ding2022snn}, StoG~\cite{ding2024enhancing}, DLIF~\cite{dingrobust}, SR~\cite{liu2024enhancing} and FEEL~\cite{xu2024feel}, under both vanilla and adversarial training (AT) settings. As shown in Table~\ref{tb:sota}, MPD-SGR demonstrates superior robustness across various attacks and datasets, outperforming all SOTA methods. Under vanilla training, it significantly surpasses the SOTA baseline (SR) by 15.87\%, 12.00\%, and 9.57\% on CIFAR-10, and by 18.35\%, 1.03\% (PGD), and 1.98\% (BIM) on CIFAR-100 under FGSM, PGD, and BIM attacks, respectively. With AT, MPD-SGR maintains superior robustness, exceeding the SOTA baseline (FEEL) by 13.52\%, 3.86\%, and 4.22\% on CIFAR-10, and the SOTA baseline (SR) by 16.35\%, 5.16\%, and 3.17\% on CIFAR-100 for FGSM, PGD, and BIM attacks, respectively.

Current adversarial defense mechanisms enhance model robustness against perturbations, yet often incur a degradation in clean accuracy—a trade-off particularly pronounced in AT frameworks and consistent with findings in prior literature~\cite{wu2024rsc}. While SR achieves relatively strong robustness under AT, its clean accuracy declines markedly to 61.26\% on CIFAR-100 and 88.26\% on CIFAR-10, indicating a significant compromise in nominal performance. In contrast, MPD-SGR achieves a more favorable balance, demonstrating its practical advantage.

\subsubsection{White-box attack.}
We compare the classification accuracy of MPD-SGR with the vanilla method (REG/RAT) under various white-box attacks across VGG11/WRN16 networks and CIFAR-10/100 datasets. As summarized in Table~\ref{tb:wb}, the MPD-SGR method significantly enhances the robustness of the vanilla SNN model against all tested attacks, with consistent gains observed under both BPTT and BPTR methods across all datasets and architectures. Notably, without AT, the vanilla SNN completely fails under strong iterative attacks such as PGD and BIM, with accuracy dropping close to 0\%. In contrast, the MPD-SGR method maintains a certain level of performance, achieving an accuracy of approximately 10\%.

We evaluate how model accuracy degrades under increasing PGD attack intensity $\epsilon$ and number of iterations $k$ on the CIFAR-10 dataset. The results in Figure~\ref{fg:wb-plus} show that the MPD-SGR exhibits a slower accuracy decline than the baseline under stronger PGD attacks for both VGG and WRN architectures, regardless of whether AT is used. This demonstrates that our method exhibits strong tolerance to more intense adversarial perturbations. For additional results about white-box attacks, including varying $T$ and different attack settings, please refer to Appendix~E.

\subsubsection{Black-box attack.}
We train a substitute SNN model to generate adversarial examples that are transferable to both the vanilla and MPD-SGR models for black-box attacks. This substitute SNN is trained without any regularization on the same dataset, using identical neuron parameters and architecture as the target model. Figure~\ref{fg:bb} highlights the performance of the MPD-SGR method across various black-box attacks. We observe that black-box attacks are less effective than white-box attacks, and the MPD-SGR method also achieves comparable robustness improvements under black-box attacks as in the white-box setting, suggesting its robustness stems not from gradient obfuscation but from inherent properties of the method. For detailed gradient obfuscation analysis, please refer to the Appendix~F.

\begin{table}[t]
    \setlength{\tabcolsep}{1.5mm}
    \centering
    \begin{tabular}{cccccc}
        \hline
        Model+SG                               & Methods & Clean          & FSGM           & PGD            & BIM            \\ \hline
        \multirow{2}{*}{\centering VGG11+Rec}  & REG     & \textbf{91.85} & 24.00          & 3.13           & 2.33           \\
                                               & Ours    & 91.23          & \textbf{43.28} & \textbf{15.82} & \textbf{14.2}  \\
        \multirow{2}{*}{\centering VGG11+Sig}  & REG     & \textbf{92.15} & 19.42          & 0.24           & 0.15           \\
                                               & Ours    & 89.38          & \textbf{37.25} & \textbf{9.26}  & \textbf{7.23}  \\
        \multirow{2}{*}{\centering VGG11+Sup}  & REG     & \textbf{86.82} & 21.39          & 0.82           & 0.50           \\
                                               & Ours    & 84.45          & \textbf{43.42} & \textbf{6.32}  & \textbf{4.50}  \\ \hline
        \multirow{2}{*}{\centering WRN16+Rec}  & REG     & \textbf{92.94} & 24.13          & 0.08           & 0.09           \\
                                               & Ours    & 92.13          & \textbf{49.51} & \textbf{18.15} & \textbf{13.83} \\
        \multirow{2}{*}{\centering WRN16+Sig}  & REG     & \textbf{92.50} & 16.80          & 0.02           & 0.01           \\
                                               & Ours    & 91.44          & \textbf{36.30} & \textbf{8.77}  & \textbf{7.05}  \\
        \multirow{2}{*}{\centering  WRN16+Sup} & REG     & \textbf{86.57} & 11.54          & 0.05           & 0.05           \\
                                               & Ours    & 86.54          & \textbf{39.25} & \textbf{7.17}  & \textbf{4.74}  \\ \hline
    \end{tabular}
    \caption{Performance (\%) of the proposed MPD-SGR method under three different SG functions.}  \label{tb:sg_cpr}
\end{table}

\begin{table}[t]
    \centering
    \setlength{\tabcolsep}{1mm}

    \begin{tabular}{cccccc}
        \toprule
        Dataset & Model & Methods  & Clean          & FGSM           & PGD            \\
        \midrule
        \multirow{6}{*}{\makecell{Tiny-                                               \\ImageNet}} & VGG16          & DIR              & \textbf{57.90}          & 2.04          & 0.01         \\
                & VGG16 & DIR+Ours & 54.78          & \textbf{14.33} & \textbf{5.72}  \\
                & VGG16 & POS      & \textbf{48.14} & 6.79           & 2.68           \\
                & VGG16 & POS+Ours & 47.83          & \textbf{20.42} & \textbf{8.21}  \\
                & VGG16 & RSC      & \textbf{47.47} & 22.63          & 13.75          \\
                & VGG16 & RSC+Ours & 46.98          & \textbf{35.06} & \textbf{17.60} \\
        \bottomrule
    \end{tabular}
    \caption{Performance (\%) of the proposed MPD-SGR method under three different spike coding methods.} \label{tab:atk_coding}
\end{table}

\subsubsection{Non-gradient attack.}
We further evaluate the impact of MPD-SGR on the robustness of SNNs under non-gradient attacks, specifically two types of random perturbations: Gaussian ($\delta \sim \mathcal{N}(0, \epsilon)$) and Uniform ($\delta \sim \mathcal{U}(-\epsilon, \epsilon)$) noise. All results are reported using a fixed noise intensity of $\epsilon=0.1$. As shown in Table~\ref{tb:atk_nograd}, our method maintains high classification performance under both perturbation types, significantly outperforming the REG baseline, which suffers severe performance degradation (accuracy drops of 48.09\% and 23.86\%). Although SR exhibits a degree of robustness, its substantially lower clean accuracy (66.76\%) fundamentally constrains its performance under noisy conditions, resulting in absolute accuracies below those achieved by our approach. These results indicate that MPD-SGR exhibits strong robustness to non-gradient perturbations, showcasing its versatility across diverse input disturbances.

\subsection{Ablation study}
\subsubsection{Performance under different SG functions.}
We investigate the applicability of our method to other commonly used SG functions, including Rectangular (Rec)~\cite{wu2019direct}, Sigmoid (Sig)~\cite{roy2019scaling} and Superspike (Sup)~\cite{zenke2018superspike}. As presented in Table~\ref{tb:sg_cpr}, our proposed MPD-SGR method consistently enhances robustness across different SG functions, further confirming the versatility and effectiveness of our method.

\subsubsection{Performance under different spike coding.}
Considering the significant impact of input spike encoding methods on the robustness of SNNs, we also compare the performance of MPD-SGR with that of the vanilla SNN model under different spike coding methods, including Direct Coding (DIR), Poisson Coding (POS) and Randomized Smoothing Coding (RSC)~\cite{wu2024rsc}. The proposed MPD-SGR method exhibits consistently superior performance over the baseline SNN model across all tested spike coding schemes (Table~\ref{tab:atk_coding}), demonstrating the compatibility of our method with various spike coding approaches.

\section{Conclusion}
This paper presents a theoretical and empirical investigation into the role of neuronal membrane potential distribution (MPD) and surrogate gradient (SG) mechanisms in enhancing the adversarial robustness of deep spiking neural networks (SNNs). Through rigorous gradient analysis, we establish a theoretical framework that formally connects input perturbation sensitivity to the magnitude characteristics of SG in SNNs. Our analysis reveals that reducing the proportion of membrane potentials within the gradient-available range of the SG function effectively mitigates the sensitivity of SNNs to input perturbations. Based on these theoretical insights, we propose MPD-driven surrogate gradient regularization (MPD-SGR), which systematically constrains the MPD during training to optimize gradient flow, thereby enhancing the network's resilience against adversarial attacks. Extensive experiments across multiple benchmark datasets and SNN architectures show that MPD-SGR maintains comparable clean accuracy while consistently improving robustness against gradient-based attacks and random noise, outperforming existing defense methods. The proposed approach also demonstrates remarkable adaptability across diverse network configurations, SG function variants, and spike coding schemes. These results position MPD-SGR as a promising, generalizable defense strategy for SNNs under diverse perturbation scenarios.

\section{Acknowledgements}
This work was supported by the National Natural Science Foundation of China (Grant Nos. 62236007, 32441113, and 62276235). The authors would also like to acknowledge anonymous reviewers and chairs for providing insightful comments to help improve this work.

\bibliography{19424}

@inproceedings{bu2023rate,
  title     = {Rate gradient approximation attack threats deep spiking neural networks},
  author    = {Bu, Tong and Ding, Jianhao and Hao, Zecheng and Yu, Zhaofei},
  booktitle = {Proceedings of the IEEE/CVF Conference on Computer Vision and Pattern Recognition},
  pages     = {7896--7906},
  year      = {2023}
}

@article{ding2022snn,
  title   = {Snn-rat: Robustness-enhanced spiking neural network through regularized adversarial training},
  author  = {Ding, Jianhao and Bu, Tong and Yu, Zhaofei and Huang, Tiejun and Liu, Jian},
  journal = {Advances in Neural Information Processing Systems},
  volume  = {35},
  pages   = {24780--24793},
  year    = {2022}
}

@article{simonyan2014very,
  title   = {Very deep convolutional networks for large-scale image recognition},
  author  = {Simonyan, Karen and Zisserman, Andrew},
  journal = {arXiv preprint arXiv:1409.1556},
  year    = {2014}
}

@article{zagoruyko2016wide,
  title   = {Wide residual networks},
  author  = {Zagoruyko, Sergey},
  journal = {arXiv preprint arXiv:1605.07146},
  year    = {2016}
}

@article{goodfellow2014explaining,
  title   = {Explaining and harnessing adversarial examples},
  author  = {Goodfellow, Ian J and Shlens, Jonathon and Szegedy, Christian},
  journal = {arXiv preprint arXiv:1412.6572},
  year    = {2014}
}

@article{madry2017towards,
  title   = {Towards deep learning models resistant to adversarial attacks},
  author  = {Madry, Aleksander},
  journal = {arXiv preprint arXiv:1706.06083},
  year    = {2017}
}

@incollection{kurakin2018adversarial,
  title     = {Adversarial examples in the physical world},
  author    = {Kurakin, Alexey and Goodfellow, Ian J and Bengio, Samy},
  booktitle = {Artificial intelligence safety and security},
  pages     = {99--112},
  year      = {2018},
  publisher = {Chapman and Hall/CRC}
}

@inproceedings{carlini2017towards,
  title        = {Towards evaluating the robustness of neural networks},
  author       = {Carlini, Nicholas and Wagner, David},
  booktitle    = {2017 IEEE symposium on security and privacy (sp)},
  pages        = {39--57},
  year         = {2017},
  organization = {IEEE}
}

@inproceedings{ding2024enhancing,
  title     = {Enhancing the robustness of spiking neural networks with stochastic gating mechanisms},
  author    = {Ding, Jianhao and Yu, Zhaofei and Huang, Tiejun and Liu, Jian K},
  booktitle = {Proceedings of the AAAI Conference on Artificial Intelligence},
  pages     = {492--502},
  year      = {2024}
}

@inproceedings{xu2024feel,
  title     = {FEEL-SNN: Robust Spiking Neural Networks with Frequency Encoding and Evolutionary Leak Factor},
  author    = {Xu, Mengting and Ma, De and Tang, Huajin and Zheng, Qian and Pan, Gang},
  booktitle = {The Thirty-eighth Annual Conference on Neural Information Processing Systems},
  year      = {2024}
}

@inproceedings{liu2024enhancing,
  title        = {Enhancing Adversarial Robustness in SNNs with Sparse Gradients},
  author       = {Liu, Yujia and Bu, Tong and Ding, Jianhao and Hao, Zecheng and Huang, Tiejun and Yu, Zhaofei},
  booktitle    = {International Conference on Machine Learning},
  pages        = {30738--30754},
  year         = {2024},
  organization = {PMLR}
}

@article{wu2018spatio,
  title     = {Spatio-temporal backpropagation for training high-performance spiking neural networks},
  author    = {Wu, Yujie and Deng, Lei and Li, Guoqi and Zhu, Jun and Shi, Luping},
  journal   = {Frontiers in neuroscience},
  volume    = {12},
  pages     = {331},
  year      = {2018},
  publisher = {Frontiers Media SA}
}

@inproceedings{wu2019direct,
  title     = {Direct training for spiking neural networks: Faster, larger, better},
  author    = {Wu, Yujie and Deng, Lei and Li, Guoqi and Zhu, Jun and Xie, Yuan and Shi, Luping},
  booktitle = {Proceedings of the AAAI conference on artificial intelligence},
  volume    = {33},
  number    = {01},
  pages     = {1311--1318},
  year      = {2019}
}

@inproceedings{roy2019scaling,
  title        = {Scaling deep spiking neural networks with binary stochastic activations},
  author       = {Roy, Deboleena and Chakraborty, Indranil and Roy, Kaushik},
  booktitle    = {2019 IEEE international conference on cognitive computing (ICCC)},
  pages        = {50--58},
  year         = {2019},
  organization = {IEEE}
}

@article{zenke2018superspike,
  title     = {Superspike: Supervised learning in multilayer spiking neural networks},
  author    = {Zenke, Friedemann and Ganguli, Surya},
  journal   = {Neural computation},
  volume    = {30},
  number    = {6},
  pages     = {1514--1541},
  year      = {2018},
  publisher = {MIT Press}
}

@inproceedings{gu2019stca,
  title     = {STCA: Spatio-temporal credit assignment with delayed feedback in deep spiking neural networks.},
  author    = {Gu, Pengjie and Xiao, Rong and Pan, Gang and Tang, Huajin},
  booktitle = {IJCAI},
  volume    = {15},
  pages     = {1366--1372},
  year      = {2019}
}

@article{neftci2019surrogate,
  title     = {Surrogate gradient learning in spiking neural networks: Bringing the power of gradient-based optimization to spiking neural networks},
  author    = {Neftci, Emre O and Mostafa, Hesham and Zenke, Friedemann},
  journal   = {IEEE Signal Processing Magazine},
  volume    = {36},
  number    = {6},
  pages     = {51--63},
  year      = {2019},
  publisher = {IEEE}
}

@inproceedings{marchisio2020spiking,
  title        = {Is spiking secure? a comparative study on the security vulnerabilities of spiking and deep neural networks},
  author       = {Marchisio, Alberto and Nanfa, Giorgio and Khalid, Faiq and Hanif, Muhammad Abdullah and Martina, Maurizio and Shafique, Muhammad},
  booktitle    = {2020 International Joint Conference on Neural Networks (IJCNN)},
  pages        = {1--8},
  year         = {2020},
  organization = {IEEE}
}

@inproceedings{sharmin2019comprehensive,
  title        = {A comprehensive analysis on adversarial robustness of spiking neural networks},
  author       = {Sharmin, Saima and Panda, Priyadarshini and Sarwar, Syed Shakib and Lee, Chankyu and Ponghiran, Wachirawit and Roy, Kaushik},
  booktitle    = {2019 International Joint Conference on Neural Networks (IJCNN)},
  pages        = {1--8},
  year         = {2019},
  organization = {IEEE}
}

@inproceedings{el2021securing,
  title        = {Securing deep spiking neural networks against adversarial attacks through inherent structural parameters},
  author       = {El-Allami, Rida and Marchisio, Alberto and Shafique, Muhammad and Alouani, Ihsen},
  booktitle    = {2021 Design, Automation \& Test in Europe Conference \& Exhibition (DATE)},
  pages        = {774--779},
  year         = {2021},
  organization = {IEEE}
}

@article{chowdhury2021towards,
  title     = {Towards understanding the effect of leak in spiking neural networks},
  author    = {Chowdhury, Sayeed Shafayet and Lee, Chankyu and Roy, Kaushik},
  journal   = {Neurocomputing},
  volume    = {464},
  pages     = {83--94},
  year      = {2021},
  publisher = {Elsevier}
}

@inproceedings{sharmin2020inherent,
  title        = {Inherent adversarial robustness of deep spiking neural networks: Effects of discrete input encoding and non-linear activations},
  author       = {Sharmin, Saima and Rathi, Nitin and Panda, Priyadarshini and Roy, Kaushik},
  booktitle    = {Computer Vision--ECCV 2020: 16th European Conference, Glasgow, UK, August 23--28, 2020, Proceedings, Part XXIX 16},
  pages        = {399--414},
  year         = {2020},
  organization = {Springer}
}

@inproceedings{kim2022rate,
  title        = {Rate coding or direct coding: Which one is better for accurate, robust, and energy-efficient spiking neural networks?},
  author       = {Kim, Youngeun and Park, Hyoungseob and Moitra, Abhishek and Bhattacharjee, Abhiroop and Venkatesha, Yeshwanth and Panda, Priyadarshini},
  booktitle    = {ICASSP 2022-2022 IEEE International Conference on Acoustics, Speech and Signal Processing (ICASSP)},
  pages        = {71--75},
  year         = {2022},
  organization = {IEEE}
}

@article{ma2023exploiting,
  title     = {Exploiting noise as a resource for computation and learning in spiking neural networks},
  author    = {Ma, Gehua and Yan, Rui and Tang, Huajin},
  journal   = {Patterns},
  volume    = {4},
  number    = {10},
  year      = {2023},
  publisher = {Elsevier}
}

@inproceedings{kundu2021hire,
  title     = {HIRE-SNN: Harnessing the adversarial robustness of energy-efficient deep spiking neural networks via training with crafted input noise},
  author    = {Kundu, Souvik and Pedram, Massoud},
  booktitle = {ICCV},
  year      = {2021}
}

@article{maass1997networks,
  title     = {Networks of spiking neurons: the third generation of neural network models},
  author    = {Maass, Wolfgang},
  journal   = {Neural networks},
  volume    = {10},
  number    = {9},
  pages     = {1659--1671},
  year      = {1997},
  publisher = {Elsevier}
}

@article{hu2024toward,
  title   = {Toward Large-scale Spiking Neural Networks: A Comprehensive Survey and Future Directions},
  author  = {Hu, Yangfan and Zheng, Qian and Li, Guoqi and Tang, Huajin and Pan, Gang},
  journal = {arXiv preprint arXiv:2409.02111},
  year    = {2024}
}

@article{tavanaei2019deep,
  title     = {Deep learning in spiking neural networks},
  author    = {Tavanaei, Amirhossein and Ghodrati, Masoud and Kheradpisheh, Saeed Reza and Masquelier, Timoth{\'e}e and Maida, Anthony},
  journal   = {Neural networks},
  volume    = {111},
  pages     = {47--63},
  year      = {2019},
  publisher = {Elsevier}
}

@article{zhu2024autonomous,
  title   = {Autonomous driving with spiking neural networks},
  author  = {Zhu, Rui-Jie and Wang, Ziqing and Gilpin, Leilani and Eshraghian, Jason},
  journal = {Advances in Neural Information Processing Systems},
  volume  = {37},
  pages   = {136782--136804},
  year    = {2024}
}

@article{Kim_Venkatesha_Panda_2022,
  title   = {PrivateSNN: Privacy-Preserving Spiking Neural Networks},
  volume  = {36},
  url     = {https://ojs.aaai.org/index.php/AAAI/article/view/20005},
  number  = {1},
  journal = {Proceedings of the AAAI Conference on Artificial Intelligence},
  author  = {Kim, Youngeun and Venkatesha, Yeshwanth and Panda, Priyadarshini},
  year    = {2022},
  month   = {Jun.},
  pages   = {1192-1200}
}

@article{liang2021exploring,
  title     = {Exploring adversarial attack in spiking neural networks with spike-compatible gradient},
  author    = {Liang, Ling and Hu, Xing and Deng, Lei and Wu, Yujie and Li, Guoqi and Ding, Yufei and Li, Peng and Xie, Yuan},
  journal   = {IEEE transactions on neural networks and learning systems},
  volume    = {34},
  number    = {5},
  pages     = {2569--2583},
  year      = {2021},
  publisher = {IEEE}
}

@inproceedings{guo2022reducing,
  title        = {Reducing information loss for spiking neural networks},
  author       = {Guo, Yufei and Chen, Yuanpei and Zhang, Liwen and Wang, YingLei and Liu, Xiaode and Tong, Xinyi and Ou, Yuanyuan and Huang, Xuhui and Ma, Zhe},
  booktitle    = {European Conference on Computer Vision},
  pages        = {36--52},
  year         = {2022},
  organization = {Springer}
}

@inproceedings{guo2022recdis,
  title     = {Recdis-snn: Rectifying membrane potential distribution for directly training spiking neural networks},
  author    = {Guo, Yufei and Tong, Xinyi and Chen, Yuanpei and Zhang, Liwen and Liu, Xiaode and Ma, Zhe and Huang, Xuhui},
  booktitle = {Proceedings of the IEEE/CVF conference on computer vision and pattern recognition},
  pages     = {326--335},
  year      = {2022}
}

@inproceedings{guo2023rmp,
  title     = {Rmp-loss: Regularizing membrane potential distribution for spiking neural networks},
  author    = {Guo, Yufei and Liu, Xiaode and Chen, Yuanpei and Zhang, Liwen and Peng, Weihang and Zhang, Yuhan and Huang, Xuhui and Ma, Zhe},
  booktitle = {Proceedings of the IEEE/CVF International Conference on Computer Vision},
  pages     = {17391--17401},
  year      = {2023}
}

@article{li2021differentiable,
  title   = {Differentiable spike: Rethinking gradient-descent for training spiking neural networks},
  author  = {Li, Yuhang and Guo, Yufei and Zhang, Shanghang and Deng, Shikuang and Hai, Yongqing and Gu, Shi},
  journal = {Advances in Neural Information Processing Systems},
  volume  = {34},
  pages   = {23426--23439},
  year    = {2021}
}

@article{zenke2021remarkable,
  title     = {The remarkable robustness of surrogate gradient learning for instilling complex function in spiking neural networks},
  author    = {Zenke, Friedemann and Vogels, Tim P},
  journal   = {Neural computation},
  volume    = {33},
  number    = {4},
  pages     = {899--925},
  year      = {2021},
  publisher = {MIT Press One Rogers Street, Cambridge, MA 02142-1209, USA journals-info~…}
}

@inproceedings{lian2023learnable,
  title     = {Learnable Surrogate Gradient for Direct Training Spiking Neural Networks.},
  author    = {Lian, Shuang and Shen, Jiangrong and Liu, Qianhui and Wang, Ziming and Yan, Rui and Tang, Huajin},
  booktitle = {IJCAI},
  pages     = {3002--3010},
  year      = {2023}
}

@inproceedings{wang2023adaptive,
  title        = {Adaptive smoothing gradient learning for spiking neural networks},
  author       = {Wang, Ziming and Jiang, Runhao and Lian, Shuang and Yan, Rui and Tang, Huajin},
  booktitle    = {International Conference on Machine Learning},
  pages        = {35798--35816},
  year         = {2023},
  organization = {PMLR}
}

@inproceedings{dingrobust,
  title        = {Robust Stable Spiking Neural Networks},
  author       = {Ding, Jianhao and Pan, Zhiyu and Liu, Yujia and Yu, Zhaofei and Huang, Tiejun},
  booktitle    = {Forty-first International Conference on Machine Learning},
  year         = {2024},
  organization = {PMLR}
}

@article{qin2019adversarial,
  title   = {Adversarial robustness through local linearization},
  author  = {Qin, Chongli and Martens, James and Gowal, Sven and Krishnan, Dilip and Dvijotham, Krishnamurthy and Fawzi, Alhussein and De, Soham and Stanforth, Robert and Kohli, Pushmeet},
  journal = {Advances in neural information processing systems},
  volume  = {32},
  year    = {2019}
}

@inproceedings{zheng2021going,
  title     = {Going deeper with directly-trained larger spiking neural networks},
  author    = {Zheng, Hanle and Wu, Yujie and Deng, Lei and Hu, Yifan and Li, Guoqi},
  booktitle = {Proceedings of the AAAI conference on artificial intelligence},
  pages     = {11062--11070},
  year      = {2021}
}

@inproceedings{wu2024rsc,
  title     = {RSC-SNN: Exploring the Trade-off Between Adversarial Robustness and Accuracy in Spiking Neural Networks via Randomized Smoothing Coding},
  author    = {Wu, Keming and Yao, Man and Chou, Yuhong and Qiu, Xuerui and Yang, Rui and Xu, Bo and Li, Guoqi},
  booktitle = {Proceedings of the 32nd ACM International Conference on Multimedia},
  pages     = {2748--2756},
  year      = {2024}
}

@article{le2015tiny,
  title   = {Tiny imagenet visual recognition challenge},
  author  = {Le, Yann and Yang, Xuan},
  journal = {CS 231N},
  volume  = {7},
  number  = {7},
  pages   = {3},
  year    = {2015}
}

@article{kim2020torchattacks,
  title   = {Torchattacks: A pytorch repository for adversarial attacks},
  author  = {Kim, Hoki},
  journal = {arXiv preprint arXiv:2010.01950},
  year    = {2020}
}

@inproceedings{wang2022learngene,
  title     = {Learngene: From open-world to your learning task},
  author    = {Wang, Qiu-Feng and Geng, Xin and Lin, Shu-Xia and Xia, Shi-Yu and Qi, Lei and Xu, Ning},
  booktitle = {Proceedings of the AAAI Conference on Artificial Intelligence},
  volume    = {36},
  number    = {8},
  pages     = {8557--8565},
  year      = {2022}
}

@inproceedings{wang2024vision,
  title     = {Vision Transformers as Probabilistic Expansion from Learngene},
  author    = {Qiufeng Wang and Xu Yang and Haokun Chen and Xin Geng},
  booktitle = {Forty-first International Conference on Machine Learning},
  year      = {2024},
  url       = {https://openreview.net/forum?id=5ExWEazod5}
}

@inproceedings{anonymous2024clusterlearngene,
  title     = {Cluster-Learngene: Inheriting Adaptive Clusters for Vision Transformers},
  author    = {Qiufeng Wang and Xu Yang and Fu Feng and Jing Wang and Xin Geng},
  booktitle = {The Thirty-eighth Annual Conference on Neural Information Processing Systems},
  year      = {2024},
  url       = {https://openreview.net/forum?id=92vVuJVLVW}
}

\end{document}

\maketitle

% ----------- Supplementary Content Starts Here -----------

\appendix
\section{Proofs of Theorems} \label{app:proof}

\begin{theorem}
    In an iterative LIF model with decay factor $\tau$ over $T$ timesteps, the postsynaptic input to neurons processed through the tdBN layer is normalized to follow the distribution $\overline{I} \sim N(\beta_{c}, (\lambda_{c}\alpha V_{th})^2)$. For $t=1, 2,3,...,T$, the membrane potentials satisfy $\overline{U}_{c}^{l}(t) \sim N((\beta_c D(\tau, t) - S(t) , (\lambda_{c}\alpha V_{th})^2 D(\tau^2, t))$, where $c$ is the channel number of the tdBN layer, $D(\tau, t) = \sum_{i=1}^{t} \tau^{t-i}$ is the cumulative decay function and $S(t)$ is the cumulative response strength constant.
\end{theorem}
\begin{proof}
    Assuming the postsynaptic current $I_{c}$ at $c$-th input channel satisfy Gaussian distribution. The tdBN \cite{zheng2021going} normalizes it as follows:
    \begin{align}
         & \hat{I}_{c} =\frac{\alpha V_{th}\left(I_{c}-\mathbb{E}\left[I_{c}\right]\right)}{\sqrt{\mathbb{VAR}\left[I_{c}\right]+\epsilon}} \label{eq:tdbn1} \\
         & \overline{I}_{c} =\lambda_{c} \hat{I}_{c}+\beta_{c}, \label{eq:tdbn2}
    \end{align}
    Eqs \ref{eq:tdbn1} eliminates the mean of the input and normalizes its variance to 1. Subsequently, the variance is scaled by two factors, $\alpha$ and $V_{th}$, resulting in $\hat{I}_{c}$ with a distribution of $N(0, (\alpha V_{th})^2)$. Eq. \ref{eq:tdbn2} introduce the learnable affine transformation parameters $\lambda_{c}$ and $\beta_{c}$, which shift and scale the normalized input current $\hat{I}_{c}$, producing $\overline{I}_{c}$ with a distribution of $N(\beta_{c}, (\lambda_{c}\alpha V_{th})^2)$.

    $\overline{I}_{c}$ is then fed to the iterative LIF neuron following dynamics in Eq.3 in the main text. Combining Eq.2 and Eq.3 in the main text, we obtain the recursive formula for the membrane potential prior to spike generation, expressed as follows:
    \begin{equation}
        U_{c}^{l}(t) = \tau (U_{c}^{l}(t-1) - O_{c}^{l}(t-1) r_{c}^{l}(t-1)) + \overline{I}^{l-1}_{c}(t) \label{eq:m_prior}
    \end{equation}

    By stacking Eq. \ref{eq:m_prior}, we have:
    \begin{equation} \label{eq:m_prior_stack}
        \resizebox{.91\linewidth}{!}{$
                \begin{aligned}
                     & U_{c}^{l}(t) - \tau U_{c}^{l}(t-1) = \overline{I}^{l-1}_{c}(t) - \tau O_{c}^{l}(t-1) r_{c}^{l}(t-1)                   \\
                     & \tau U_{c}^{l}(t-1) - \tau^2 U_{c}^{l}(t-2) = \tau \overline{I}^{l-1}_{c}(t-1) - \tau^2 O_{c}^{l}(t-2) r_{c}^{l}(t-2)
                \end{aligned}
            $}
    \end{equation}

    Notice that the left sides of Eq. \ref{eq:m_prior_stack} can iteratively eliminate $U_{c}^{l}(t)$. With the initial membrane potential $U_{c}^{l}(0)$ set to zero, we obtain the following:
    \begin{align} \label{eq:m_prior_define}
        U_{c}^{l}(t) = \underbrace{\sum_{i=1}^{t} \tau^{t-i} \overline{I}_{c}^{l-1}(i)}_{M_{1}(t)} - \underbrace{\sum_{i=1}^{t-1} \tau^{t+1-i} O_{c}^{l}(i) r_{c}^{l}(i)}_{M_{2}(t)}
    \end{align}

    From the above formula, The distribution of membrane potential is mainly composed of two factors $M_{1}(t)$ and $M_{2}(t)$. $M_{1}(t)$ is a weighted sum of independent Gaussian variables $\overline{I}_{c}^{l-1}(i)$, so it is still a Gaussian distribution with the following distribution parameters:
    \begin{equation}
        M_{1}(t) \sim N(\beta_c \sum_{i=1}^{t} \tau^{t-i} , (\lambda_{c}\alpha V_{th})^2 \sum_{i=1}^{t} \tau^{2(t-i)})
    \end{equation}

    In $M_{2}(t)$, $r_{c}^{l}(i)$ represents the value subtracted due to reset. Under the hard reset condition, where $r_{c}^{l}(i) = U_{c}^{l}(t-1)$, Eq \ref{eq:m_prior_define} becomes recursive. To facilitate derivation, we adopt the soft reset approach, setting $r_{c}^{l}(i) = v_{th}$. Furthermore, $O_{c}^{l}(i)$ and $U_{c}^{l}(t)$ exhibit a nonlinear coupling relationship, making precise analysis challenging. Therefore, we employ mean field theory to approximate the true values using the expected value of the spikes, yielding the following approximate expression:
    \begin{align}
        M_{2}(t) & \thickapprox  \sum_{i=1}^{t-1} \tau^{t+1-i} \mathbb{E}[O_{c}^{l}(i)] v_{th}         \\
                 & = \sum_{i=1}^{t-1} \tau^{t+1-i} \mathbb{P}[\overline{U}_{c}^{l}(i) > v_{th}] v_{th} \\
                 & = \sum_{i=1}^{t-1} \tau^{t+1-i} f^l(i) v_{th} = S(t)
    \end{align}

    In this case, the item $M_{2}(t)$ can be regarded as a constant $S(t)$, representing the cumulative response strength modulated by the spike firing rate $f^l(i)$.
    Therefore, we can get the approximate distribution of membrane potential $U_{c}^{l}(t)$ as follows:
    \begin{align}
        U_{c}^{l}(t) \sim N((\beta_c D(\tau, t) - S(t) , (\lambda_{c}\alpha V_{th})^2 D(\tau^2, t))
    \end{align}
    where $D(\tau, t) = \sum_{i=1}^{t} \tau^{t-i}$ is the cumulative decay function.
\end{proof}

\begin{theorem}
    \textbf{Adversarial Sensitivity Bound From Firing State Perspective}. For a SNN represented by $f:\mathbb{R}^n \to \mathbb{R}^m$, perturbation input $x + \varepsilon \delta$ within a $\ell_p$-constrainted perturbation ball $B_p(x, \varepsilon) = \{ x + \varepsilon \delta: \|\delta\|_p \le 1 \}$ induce a finite set of firing pattern transitions. The adversarial sensitivity $S_{\mathrm{adv}}(f,x,\varepsilon)$ upper bound can be approximated as:
    \begin{equation}
        S_{\mathrm{adv}}(f,x,\varepsilon) \le \varepsilon^2 \max_{1\leq k\leq K} \| W_{\mathcal{P}_k} \|_{p \to 2}^{2}
    \end{equation}
    where $W_{\mathcal{P}_k}$ represents the affine transformation matrix determined by the firing pattern $\mathcal{P}_k = \{(\ell,i,t): \boldsymbol{\overline{U}}_i^{\ell}(t) \geq 0\}$ and $\|\cdot\|_p$ denotes the $p$-norm.
\end{theorem}
\begin{proof}
    Let $f:\mathbb{R}^n \to \mathbb{R}^m$ denote the SNN that maps the input space to the output space. For a given input $x$, the firing pattern $P(x)$ generated by the network can be defined as:
    \begin{equation}
        P(x)=\{(\ell,i, t): \boldsymbol{\overline{U}}_i^{\ell}(t) \geq 0 \}
    \end{equation}
    where $\boldsymbol{\overline{U}}_i^{\ell}(t) = \boldsymbol{U}_i^{\ell}(t; x) - v_{th}^{\ell}$. Therefore, each firing pattern $\mathcal{P}$ also corresponds uniquely to a convex polyhedral cell $C_\mathcal{P}$ in the input space:
    \begin{equation}
        C_{\mathcal{P}}=\{x\in\mathbb{R}^n
        :P(x)=\mathcal{P}\}.
    \end{equation}

    Within any $C_\mathcal{P}$, the firing pattern $\mathcal{P}$ can be considered a fixed binary gate, so the network reduces to an affine map:
    \begin{equation}
        f(x)=W_\mathcal{P} x+b_\mathcal{P}, x \in C_\mathcal{P}
    \end{equation}
    where $W_\mathcal{P} \in \mathbb{R}^{m \times n}$ and $b_\mathcal{P} \in \mathbb{R}^m$ are determined by the weights and biases under this gate configuration.

    Let the $\ell_p$-constrainted perturbation ball be $B_p(x, \varepsilon) = \{ x + \varepsilon \delta: \|\delta\|_p \le 1 \}$, where $\varepsilon$ is the perturbation budget and $\delta$ represents a perturbation direction vector. When the small perturbations $\varepsilon \delta$ that stays inside the same cell ($x + \varepsilon \delta \in C_\mathcal{P}$), the local upper bound of the output $f(x)$ is given by:
    \begin{equation}
        \| f(x+\varepsilon \delta) - f(x) \|_2^{2} = \| \varepsilon W_\mathcal{P} \delta \|_2^{2} \le \varepsilon^2 \| W_\mathcal{P}  \|_{p \to 2}^{2}
    \end{equation}
    where $\| W_\mathcal{P}  \|_{p \to 2}^{2}$ denotes the operator norm of the matrix $W_\mathcal{P}$ with respect to the input $p$-norm and output 2-norm.

    However, the adversarial sensitivity $S_\mathrm{adv}$ measures the maximum change in the output of $f$ caused by a small perturbation $\varepsilon \delta$, as follows:
    \begin{equation}
        S_{\mathrm{adv}}(f,x,\varepsilon) = \max_{\delta \in B_p(x, \varepsilon)} \bigl\| f(x+\varepsilon\delta)-f(x)\,\bigr\|_2^{2}.
    \end{equation}

    Therefore, when maximizing the "worst-case direction", we need to consider the maximum upper bound in all finite $K$ cells that intersect with the perturbation ball. The global upper bound of adversarial sensitivity is expressed as
    \begin{equation}
        S_{\mathrm{adv}}(f,x,\varepsilon) \le \varepsilon^2 \max_{1\leq k\leq K} \| W_{\mathcal{P}_k} \|_{p \to 2}^{2}
    \end{equation}

    Since the matrix $W_{\mathcal{P}_k}$ is determined by the firing pattern $\mathcal{P}_k$, it is influenced by the relationship between neuronal membrane potential and threshold. The closer the membrane potential is to the threshold, the easier it is for the neuron's firing state to flip, resulting in more convex polyhedral cells intersecting the perturbation ball and a higher upper bound. Therefore, larger margins between membrane potential and threshold shrink the set of reachable cells and typically reduce the dominating $\| W_{\mathcal{P}_k} \|_{p \to 2}^{2}$.

    In surrogate gradient training, the effective gradient magnitudes of the spike nonlinearity decrease with increasing membrane-threshold margins (cf. Eq.7 in the main text). Consequently, smaller gradient magnitudes $\frac{\boldsymbol{O}^{\ell}(t)}{\boldsymbol{U}^{\ell}(t)}$ implies larger margins $\boldsymbol{\overline{U}}^{\ell}(t)$, yielding a more stable firing pattern, which in turn reduces variability in the affine transformation $W_{\mathcal{P}_k}$ and lowers the dominating operator norm $\| W_{\mathcal{P}_k} \|_{p \to 2}^{2}$, resulting in a smaller adversarial sensitivity:
    \begin{equation}
        \frac{\boldsymbol{O}^{\ell}(t)}{\boldsymbol{U}^{\ell}(t)} \downarrow \Longrightarrow \boldsymbol{\overline{U}}^{\ell}(t) \uparrow \Longrightarrow S_{\mathrm{adv}}(f,x,\varepsilon)\downarrow.
    \end{equation}
\end{proof}

\begin{figure*}[t]
    \begin{subfigure}[b]{0.49\textwidth}
        \includegraphics[scale=0.42]{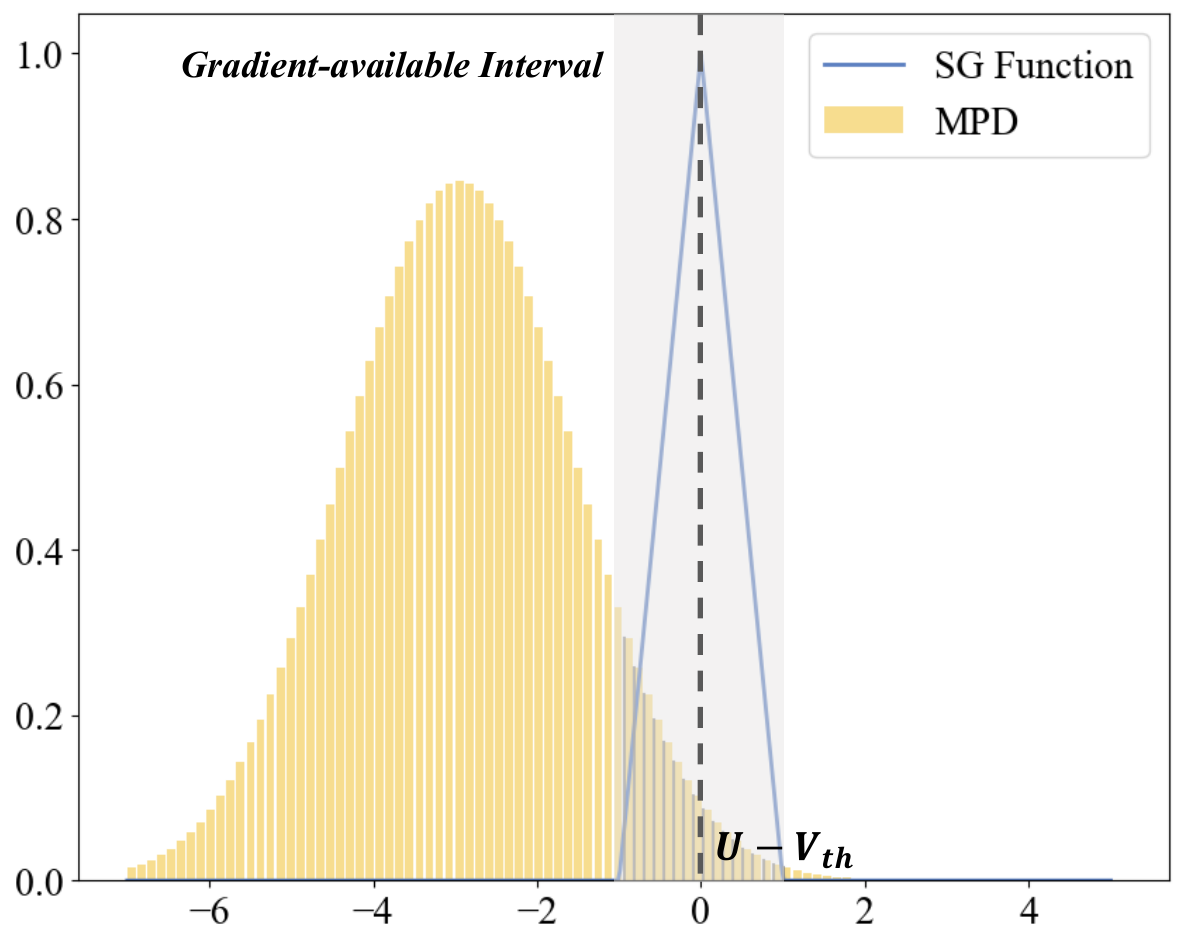}
        \caption{Degeneration}
        \label{fg:mpd_case1}
    \end{subfigure}
    \begin{subfigure}[b]{0.49\textwidth}
        \includegraphics[scale=0.42]{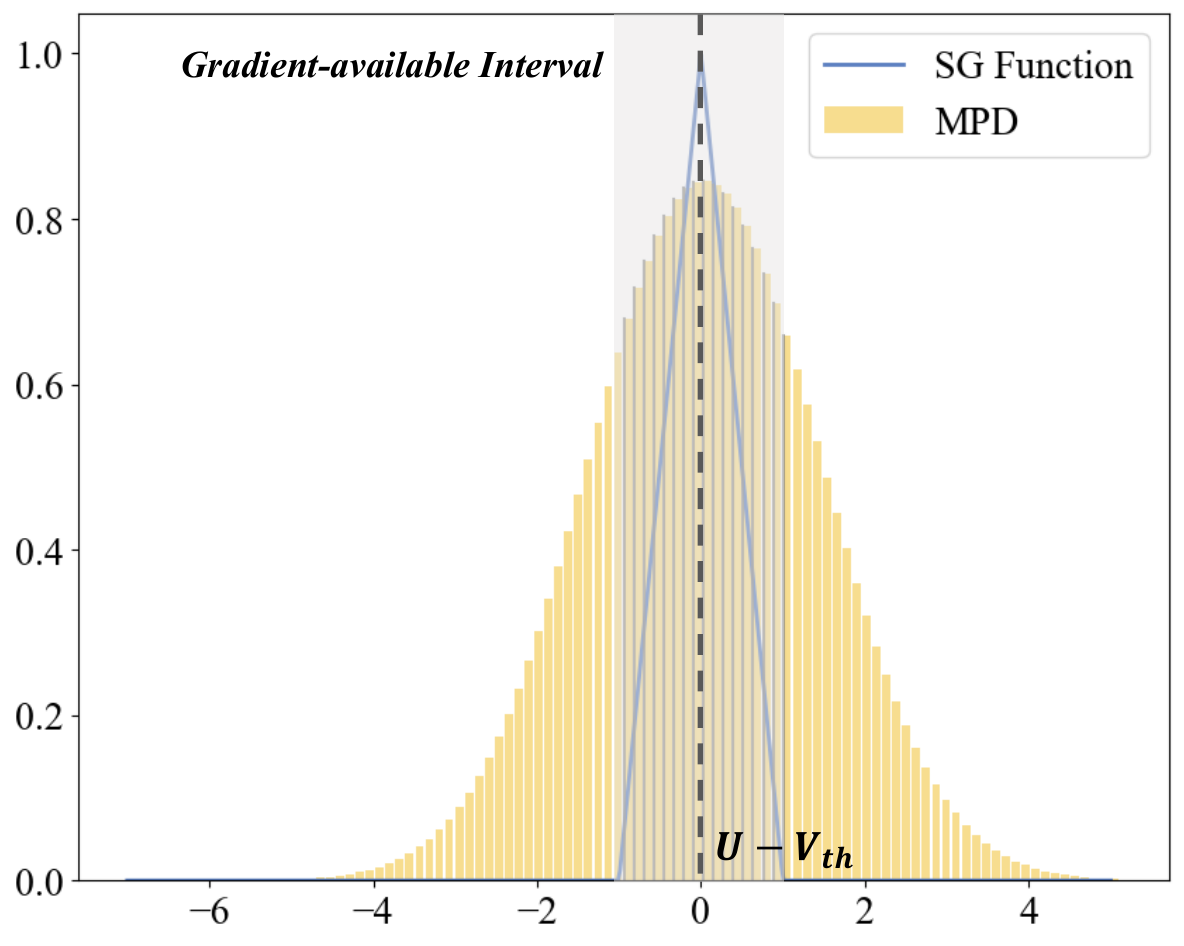}
        \caption{Saturated}
        \label{fg:mpd_case2}
    \end{subfigure}
    \caption{The relationship between the surrogate gradient (SG) function and membrane potential distribution (MPD)}
    \label{fg:mpd_case}
\end{figure*}

\section{The relationship between surrogate gradient function and membrane potential distribution}
As illustrated in Figure.\ref{fg:mpd_case}, the relationship between the surrogate gradient (SG) function and membrane potential distribution (MPD) will appear several extreme cases: (1) Degeneration (Fig.\ref{fg:mpd_case1}): almost all the membrane potential values of the neurons in a channel are below the firing threshold. Thus very few membrane potentials carry gradients, which obstructs gradient propagation; (2) Saturated (Figure.\ref{fg:mpd_case2}): Almost all the membrane potential values in a channel fall into the gradient-available interval, it is equivalent to use SGs for all gradient computation, which will enlarge the approximated errors from the accurate gradients.

\section{Analysis of MPD-SGR methods}
\subsubsection{The effect of balanced coefficient $\eta$.}
The balanced coefficient $\eta$ determines the strength of the proposed MPD-SGR effect. A larger $\eta$ reduces the proportion of membrane potentials contributing effectively to the gradient computation, obstructing gradient propagation and lowering clean accuracy. Conversely, a smaller $\eta$ imposes insufficient constraints on the surrogate gradient, failing to improve robustness effectively. Therefore, we perform an extensive exploration of the optimal coefficient $\eta$ to strike a trade-off between accuracy on clean data (clean accuracy) and robustness on adversarial inputs (adversarial accuracy). The investigation focuses on the CIFAR-10 dataset, with an SNN version of the VGG11 architecture and a timestep setting of $T=2$.

To assess the impact of $\eta$ varying within the range of 0.01 to 0.1, we conduct experiments under BPTT-based FGSM and PGD attacks (Figure~\ref{fg:eta}). We also provide the clean and adversarial accuracy of the vanilla SNN model as a baseline reference, as indicated by the dashed line in Figure~\ref{fg:eta}. It can be seen that the optimal value of $\eta$ appears to be around 0.05, striking a suitable balance between accuracy and robustness.

\begin{figure}[t]
    \includegraphics[width=0.47\textwidth]{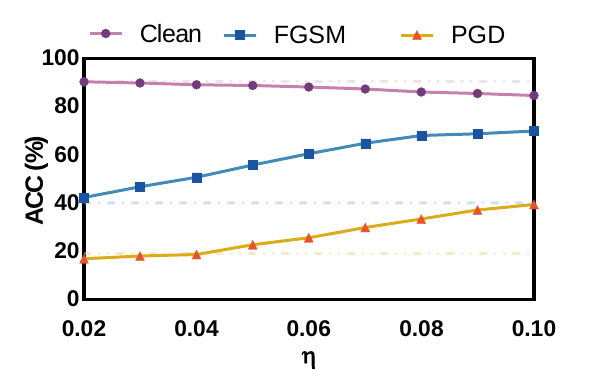}
    \caption{The influence of the coefficient parameter $\eta$ on clean and adversarial accuracy under BPTT-based attacks. The dashed line indicates the accuracy of the vanilla SNN model}
    \label{fg:eta}
\end{figure}

\begin{figure*}[!htbp]
    \includegraphics[width=1\textwidth]{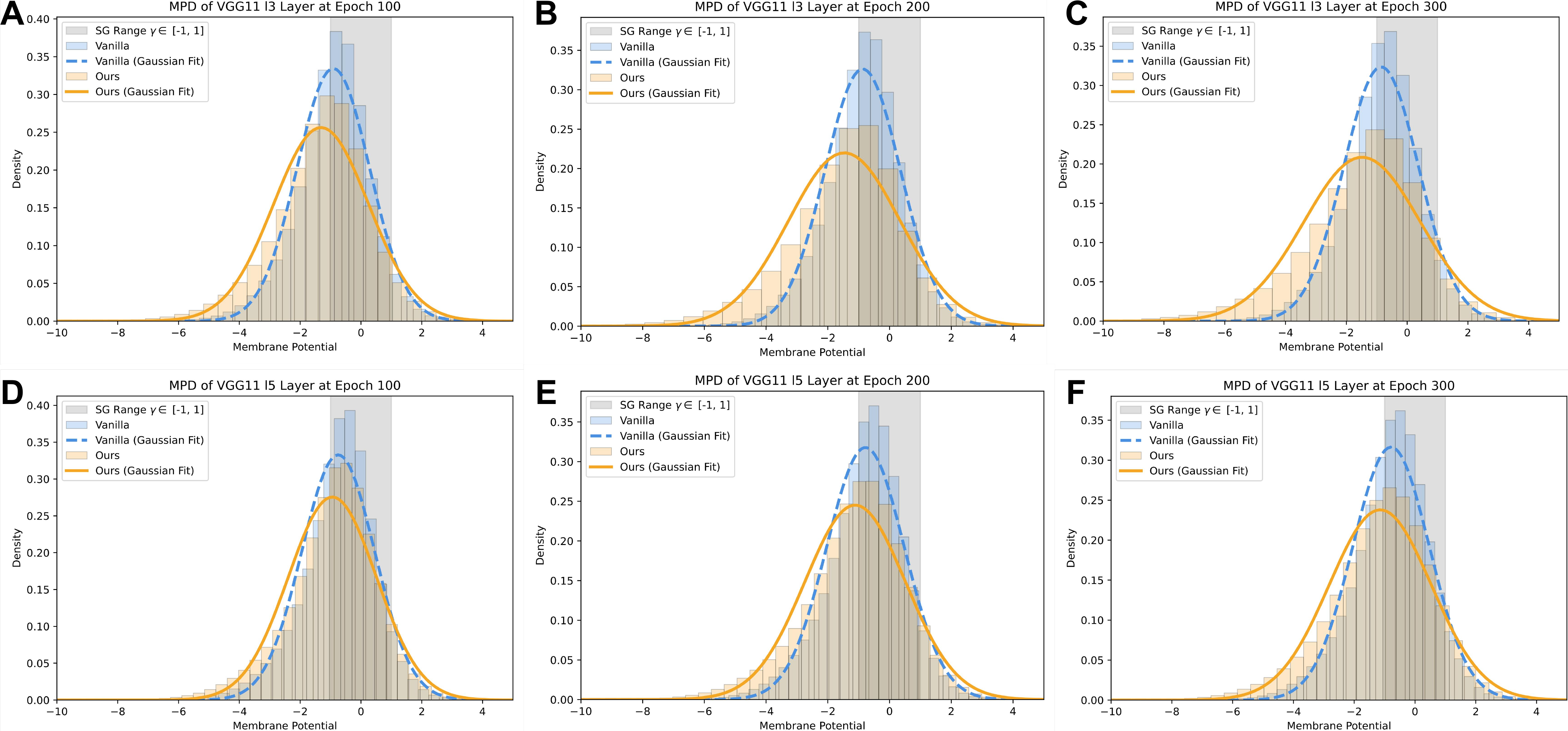}
    \caption{Membrane potential distribution (MPD) and Gaussian fitting during training. Subfigures (A-C) show the MPD histograms and corresponding Gaussian fits for the third layer of VGG11 at epochs 100, 200, and 300, respectively. Subfigures (D-F) display the same for the fifth layer. The empirical MPD closely matches a Gaussian curve, supporting our MPD-SGR strategy based on mean and variance control.}
    \label{fg:mpd_hist}
\end{figure*}

\begin{figure*}[!htbp]
    \includegraphics[width=1\textwidth]{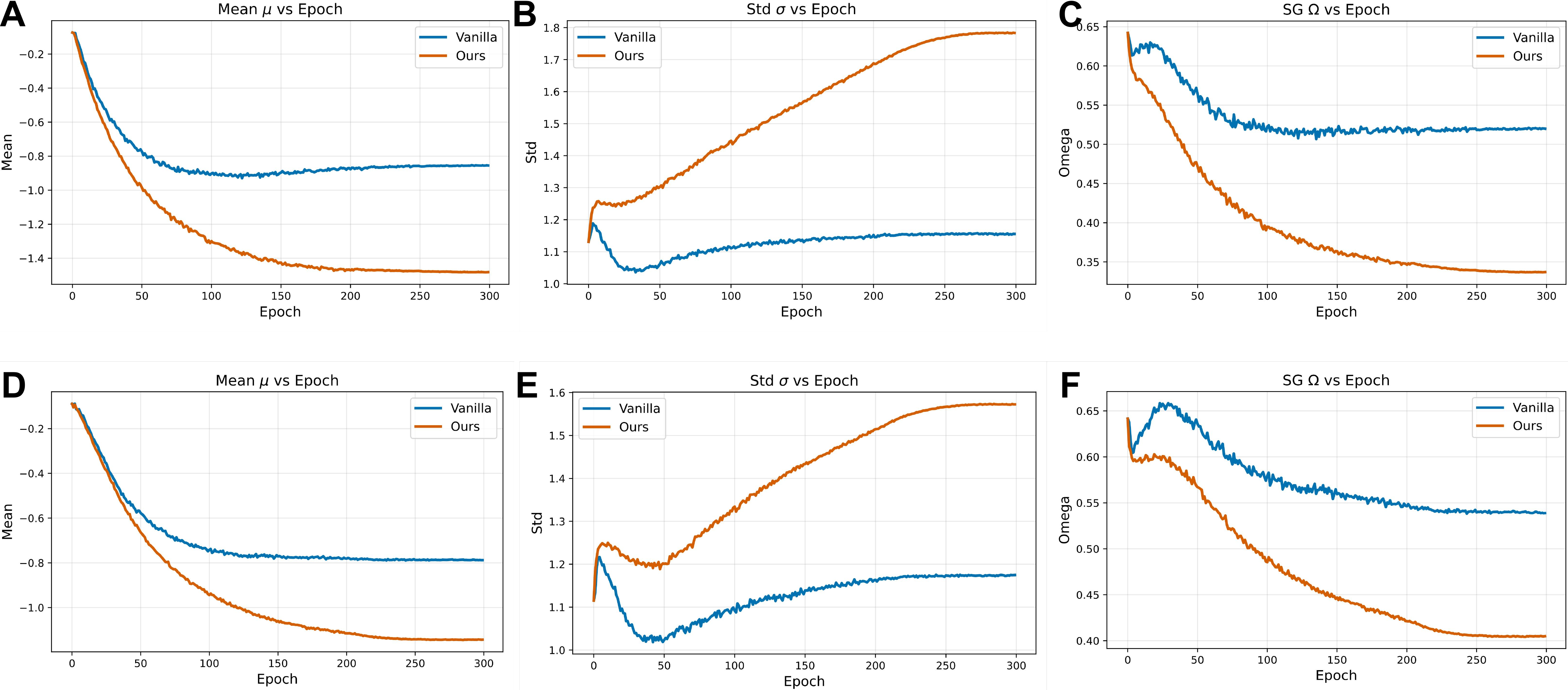}
    \caption{Evolution of MPD statistics and the regularization term $\Omega$ during training. Subfigures (A-C) show the changes in the mean, standard deviation, and $\Omega$ for the third layer of VGG11, respectively, while (D-F) present the same for the fifth layer. }
    \label{fg:mpd_curve}
\end{figure*}

\subsubsection{Analysis of MPD and SG overlap evolution}
To empirically validate our theoretical method, we visualized the membrane potential distribution (MPD) of neurons during training, along with its Gaussian fitting results, as shown in Figure~\ref{fg:mpd_hist}. The empirical MPD exhibits a close approximation to a Gaussian distribution, indicating that the membrane dynamics during training can be statistically characterized by its mean and variance. Building on this observation, the proposed MPD-SGR method explicitly regulates the mean and variance of the MPD to reduce the proportion of potentials that fall within the gradient-available interval of the surrogate gradient (SG) function (overlap $\Omega$), thereby offering empirical support for both the theoretical derivation and the design of the algorithm.

we further analyze the evolution of MPD statistics throughout training, including the mean, standard deviation, and the proposed regularization term $\Omega$, as shown in Figure~\ref{fg:mpd_curve}. The results reveal that MPD-SGR consistently reduces the mean and increases the variance of the membrane potential distribution over time. This transformation shifts more membrane potentials outside the gradient-available interval of the SG function, thereby reducing the overlap region $\Omega$. Such behavior aligns with our theoretical motivation and empirically confirms the effectiveness of MPD-SGR in promoting robustness by regulating gradient during training.

\begin{table*}[!htbp]
    \setlength{\tabcolsep}{1mm}
    \begin{center}
        \begin{sc}
            \begin{tabular}{c|c|c|c|c|c|c|c}
                \toprule
                \textbf{Dataset}           & \textbf{T}         & \textbf{Model} & \textbf{Clean}                        & \textbf{FGSM}        & \textbf{PGD}         & \textbf{BIM}         & \textbf{CW}          \\
                \midrule
                \multirow{10}{*}{CIFAR10}  & \multirow{2}{*}{8} & VGG11, BPTT    & \multirow{2}{*}{\textbf{92.49}/90.69} & 25.18/\textbf{47.59} & 0.88/\textbf{20.55}  & 0.60/\textbf{16.85}  & 7.19/\textbf{24.50}  \\
                                           &                    & VGG11, BPTR    &                                       & 21.50/\textbf{43.32} & 3.37/\textbf{27.18}  & 2.87/\textbf{26.54}  & 19.74/\textbf{37.60} \\
                \cmidrule{2-8}
                                           & \multirow{4}{*}{4} & VGG11, BPTT    & \multirow{2}{*}{\textbf{92.45}/90.62} & 23.86/\textbf{42.12} & 0.87/\textbf{17.18}  & 0.60/\textbf{14.59}  & 6.26/\textbf{18.85}  \\
                                           &                    & VGG11, BPTR    &                                       & 27.92/\textbf{47.61} & 7.82/\textbf{25.52}  & 6.61/\textbf{25.00}  & 24.38/\textbf{38.07} \\
                                           &                    & WRN16, BPTT    & \multirow{2}{*}{\textbf{93.46}/92.22} & 18.81/\textbf{45.82} & 0.02/\textbf{8.44}   & 0.03/\textbf{6.59}   & 3.61/\textbf{18.66}  \\
                                           &                    & WRN16, BPTR    &                                       & 15.28/\textbf{38.18} & 0.17/\textbf{8.03}   & 0.16/\textbf{6.65}   & 11.76/\textbf{31.21} \\
                \cmidrule{2-8}
                                           & \multirow{4}{*}{2} & VGG11, BPTT    & \multirow{2}{*}{\textbf{92.22}/90.18} & 18.17/\textbf{48.08} & 0.33/\textbf{13.03}  & 0.24/\textbf{11.17}  & 5.31/\textbf{14.63}  \\
                                           &                    & VGG11, BPTR    &                                       & 46.91/\textbf{61.69} & 15.07/\textbf{27.40} & 13.48/\textbf{26.93} & 36.43/\textbf{61.12} \\
                                           &                    & WRN16, BPTT    & \multirow{2}{*}{\textbf{92.96}/91.81} & 17.13/\textbf{48.20} & 0.00/\textbf{6.85}   & 0.01/\textbf{5.44}   & 2.27/\textbf{15.68}  \\
                                           &                    & WRN16, BPTR    &                                       & 25.70/\textbf{54.31} & 1.15/\textbf{14.93}  & 1.12/\textbf{14.16}  & 17.95/\textbf{47.59} \\
                \hline
                \multirow{10}{*}{CIFAR100} & \multirow{2}{*}{8} & VGG11, BPTT    & \multirow{2}{*}{\textbf{72.82}/70.42} & 10.14/\textbf{34.51} & 0.27/\textbf{9.03}   & 0.31/\textbf{8.41}   & 6.18/\textbf{16.12}  \\
                                           &                    & VGG11, BPTR    &                                       & 8.53/\textbf{21.25}  & 1.64/\textbf{16.66}  & 1.64/\textbf{16.44}  & 18.75/\textbf{24.96} \\
                \cmidrule{2-8}

                                           & \multirow{4}{*}{4} & VGG11, BPTT    & \multirow{2}{*}{\textbf{72.54}/69.56} & 8.95/\textbf{20.16}  & 0.24/\textbf{6.90}   & 0.20/\textbf{5.51}   & 5.62/\textbf{13.54}  \\
                                           &                    & VGG11, BPTR    &                                       & 11.45/\textbf{34.49} & 4.66/\textbf{10.07}  & 4.44/\textbf{9.94}   & 25.03/\textbf{26.99} \\
                                           &                    & WRN16, BPTT    & \multirow{2}{*}{\textbf{73.86}/72.61} & 10.99/\textbf{22.58} & 0.04/\textbf{1.97}   & 0.05/\textbf{1.77}   & 4.93/\textbf{10.34}  \\
                                           &                    & WRN16, BPTR    &                                       & 9.74/\textbf{21.60}  & 0.44/\textbf{5.52}   & 0.35/\textbf{5.50}   & 14.81/\textbf{24.13} \\
                \cmidrule{2-8}

                                           & \multirow{4}{*}{2} & VGG11, BPTT    & \multirow{2}{*}{\textbf{71.73}/69.23} & 7.29/\textbf{29.29}  & 0.13/\textbf{6.68}   & 0.10/\textbf{5.88}   & 3.12/\textbf{13.04}  \\
                                           &                    & VGG11, BPTR    &                                       & 18.48/\textbf{36.03} & 5.96/\textbf{15.81}  & 5.57/\textbf{15.25}  & 26.40/\textbf{35.41} \\
                                           &                    & WRN16, BPTT    & \multirow{2}{*}{\textbf{71.79}/71.35} & 11.04/\textbf{21.43} & 0.05/\textbf{1.23}   & 0.06/\textbf{1.00}   & 3.75/\textbf{7.74}   \\
                                           &                    & WRN16, BPTR    &                                       & 14.36/\textbf{29.69} & 2.06/\textbf{12.99}  & 1.84/\textbf{12.60}  & 15.34/\textbf{33.40} \\
                \bottomrule
            \end{tabular}
        \end{sc}
    \end{center}
    \caption{The classification accuracy (Vanilla/Ours) under white-box attacks. The attack perturbation strength is set to $\epsilon= 8/255$ for all attacks, the iterative steps $k= 7$ and step size $\alpha= 0.01$ for PGD, BIM.} \label{tb:wb}
\end{table*}

\begin{table*}[!htbp]
    \setlength{\tabcolsep}{1mm}
    \begin{center}
        \begin{sc}
            \begin{tabular}{c|c|c|c|c|c|c|c}
                \toprule
                \textbf{Dataset}           & \textbf{T}         & \textbf{Model} & \textbf{Clean}                        & \textbf{FGSM}        & \textbf{PGD}         & \textbf{BIM}         & \textbf{CW}          \\
                \midrule
                \multirow{10}{*}{CIFAR10}  & \multirow{2}{*}{8} & VGG11, BPTT    & \multirow{2}{*}{\textbf{91.41}/90.69} & 45.00/\textbf{59.27} & 22.95/\textbf{33.38} & 20.80/\textbf{32.61} & 36.04/\textbf{49.50} \\
                                           &                    & VGG11, BPTR    &                                       & 41.16/\textbf{61.01} & 23.07/\textbf{36.43} & 21.25/\textbf{35.50} & 54.93/\textbf{59.88} \\\cmidrule{2-8}
                                           & \multirow{4}{*}{4} & VGG11, BPTT    & \multirow{2}{*}{\textbf{91.08}/90.34} & 43.80/\textbf{58.71} & 20.63/\textbf{28.94} & 18.63/\textbf{24.74} & 29.87/\textbf{34.38} \\
                                           &                    & VGG11, BPTR    &                                       & 49.11/\textbf{66.17} & 35.74/\textbf{43.12} & 34.59/\textbf{41.17} & 58.73/\textbf{65.05} \\
                                           &                    & WRN16, BPTT    & \multirow{2}{*}{91.15/\textbf{91.34}} & 42.31/\textbf{63.32} & 19.93/\textbf{38.11} & 18.03/\textbf{33.13} & 29.60/\textbf{43.06} \\
                                           &                    & WRN16, BPTR    &                                       & 51.75/\textbf{70.93} & 34.44/\textbf{46.91} & 33.00/\textbf{44.79} & 58.90/\textbf{70.10} \\\cmidrule{2-8}
                                           & \multirow{4}{*}{2} & VGG11, BPTT    & \multirow{2}{*}{\textbf{90.71}/89.92} & 40.26/\textbf{55.89} & 19.18/\textbf{22.76} & 17.40/\textbf{19.24} & 28.49/\textbf{30.51} \\
                                           &                    & VGG11, BPTR    &                                       & 51.87/\textbf{66.28} & 40.70/\textbf{58.72} & 49.39/\textbf{57.05} & 68.25/\textbf{83.46} \\
                                           &                    & WRN16, BPTT    & \multirow{2}{*}{\textbf{92.96}/91.81} & 17.13/\textbf{48.20} & 0.00/\textbf{6.85}   & 0.01/\textbf{5.44}   & 2.27/\textbf{15.68}  \\
                                           &                    & WRN16, BPTR    &                                       & 25.70/\textbf{54.31} & 1.15/\textbf{14.93}  & 1.12/\textbf{14.16}  & 17.95/\textbf{47.59} \\
                \hline
                \multirow{10}{*}{CIFAR100} & \multirow{2}{*}{8} & VGG11, BPTT    & \multirow{2}{*}{69.43/\textbf{69.56}} & 19.07/\textbf{39.45} & 9.23/\textbf{22.23}  & 8.41/\textbf{19.45}  & 18.82/\textbf{22.53} \\
                                           &                    & VGG11, BPTR    &                                       & 25.44/\textbf{48.0}  & 15.74/\textbf{33.84} & 14.68/\textbf{33.29} & 30.33/\textbf{43.82} \\\cmidrule{2-8}

                                           & \multirow{4}{*}{4} & VGG11, BPTT    & \multirow{2}{*}{\textbf{68.60}/67.67} & 20.28/\textbf{36.18} & 9.97/\textbf{18.23}  & 8.82/\textbf{15.70}  & 15.46/\textbf{22.33} \\
                                           &                    & VGG11, BPTR    &                                       & 29.78/\textbf{40.42} & 22.14/\textbf{26.38} & 21.27/\textbf{25.88} & 40.96/\textbf{47.44} \\

                                           &                    & WRN16, BPTT    & \multirow{2}{*}{69.29/\textbf{69.67}} & 25.35/\textbf{39.47} & 12.09/\textbf{21.71} & 11.30/\textbf{18.87} & 20.17/\textbf{27.55} \\
                                           &                    & WRN16, BPTR    &                                       & 31.99/\textbf{44.23} & 22.58/\textbf{35.67} & 22.04/\textbf{35.08} & 43.23/\textbf{53.21} \\\cmidrule{2-8}

                                           & \multirow{4}{*}{2} & VGG11, BPTT    & \multirow{2}{*}{\textbf{67.79}/66.90} & 18.94/\textbf{35.10} & 9.00/\textbf{15.89}  & 8.27/\textbf{13.70}  & 15.32/\textbf{20.56} \\
                                           &                    & VGG11, BPTR    &                                       & 35.10/\textbf{42.65} & 28.79/\textbf{37.58} & 27.96/\textbf{37.22} & 42.40/\textbf{51.55} \\

                                           &                    & WRN16, BPTT    & \multirow{2}{*}{67.04/\textbf{67.69}} & 24.51/\textbf{35.51} & 11.31/\textbf{18.69} & 10.49/\textbf{17.06} & 21.40/\textbf{24.96} \\
                                           &                    & WRN16, BPTR    &                                       & 35.93/\textbf{49.97} & 28.71/\textbf{45.50} & 27.88/\textbf{44.21} & 45.07/\textbf{57.55} \\

                \bottomrule
            \end{tabular}
        \end{sc}
    \end{center}
    \caption{The classification accuracy (Vanilla/Ours with AT) under white-box attacks. The attack perturbation strength is set to $\epsilon= 8/255$ for all attacks, the iterative steps $k= 7$ and step size $\alpha= 0.01$ for PGD, BIM.} \label{tb:wb_at}
\end{table*}

\begin{figure*}[t]
    \centering
    \begin{subfigure}[b]{0.47\textwidth}
        \includegraphics[scale=0.6]{WB}
        \caption{CIFAR10}
        \label{fg:wb1}
    \end{subfigure}
    \hfill
    \begin{subfigure}[b]{0.47\textwidth}
        \includegraphics[scale=0.58]{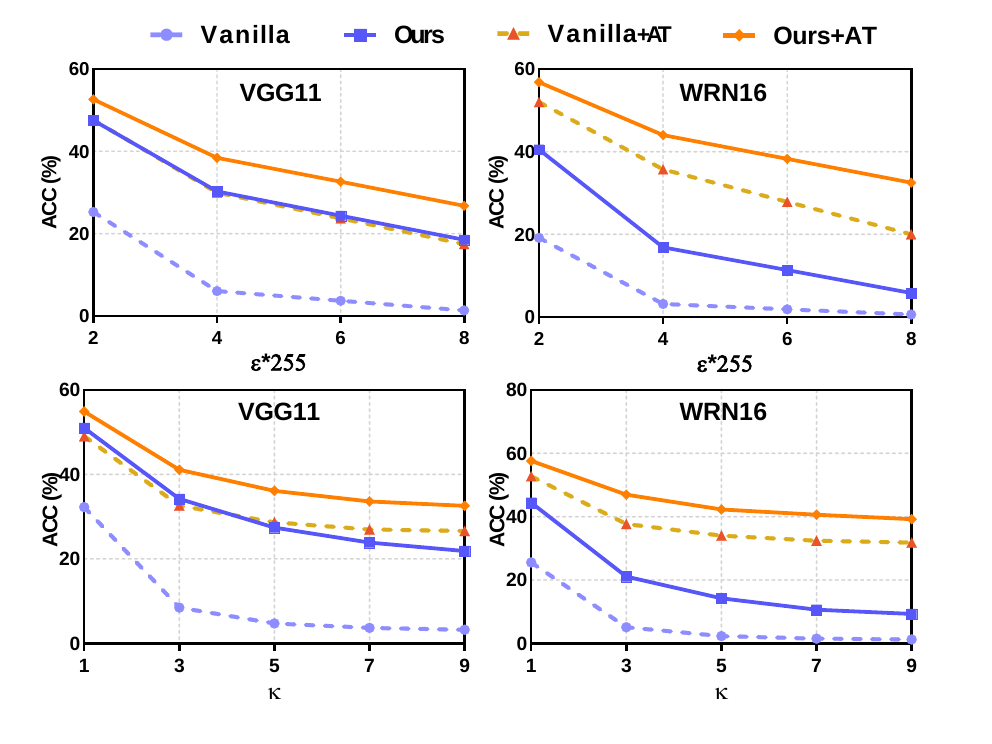}
        \caption{CIFAR100}
        \label{fg:wb2}
    \end{subfigure}
    \caption{Performance of the white-box PGD attack with increasing perturbation $\epsilon$ and iterative step $k$ = 4 (Top Panels), increasing iterative step $k$ and $\epsilon$ = 8/255 (Bottom Panels).}
    \label{fg:wb-plus}
\end{figure*}

\begin{figure}[t]
    \centering
    \includegraphics[width=0.5\textwidth]{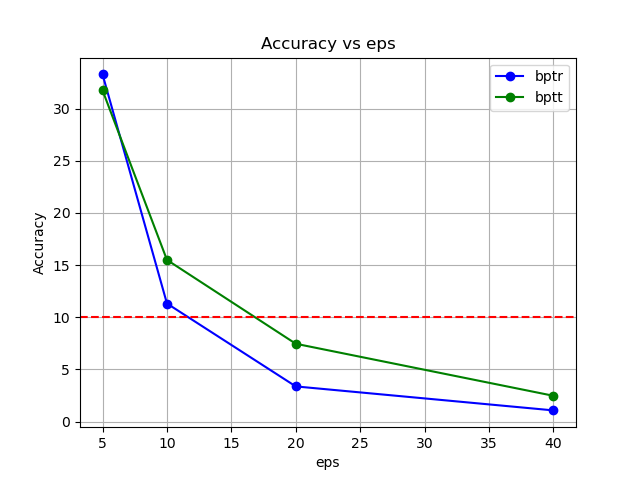}
    \caption{Performance of untargeted PGD attacks with different approximation. The red dashed line represents the accuracy of random guess (10\% for CIFAR-10).}
    \label{fg:unbound_attack}
\end{figure}

\begin{table*}[t]
    \caption{Checklist for characteristic behaviors caused by obfuscated and masked gradients.}
    \label{tb:checklist}
    \begin{center}
        \begin{tabular}{|p{12cm}|c|c|c|}
            \hline
            \textbf{Items to identify gradient obfuscation}                      & \textbf{BPTR} & \textbf{BPTT} \\
            \hline
            (1) Single-step attack performs better compared to iterative attacks & Pass          & Pass          \\
            \hline
            (2) Black-box attacks performs better compared to white-box attacks  & Pass          & Pass          \\
            \hline
            (3) Increasing perturbation bound can't increase attack strength     & Pass          & Pass          \\
            \hline
            (4) Unbounded attacks can't reach $\sim 100\%$ success               & Pass          & Pass          \\
            \hline
            (5) Adversarial example can be found through random sampling         & Pass          & Pass          \\
            \hline
        \end{tabular}
    \end{center}
\end{table*}

\section{Details of Implementation}
\subsection{Experimental settings}
The training process lasts for 300 epochs and all reported results are averaged over five independent algorithm runs. In the training process, the stochastic gradient descent (sgd) is deployed, and the initial learning rate is set to 0.1. The learning rate uses a cosine annealing schedule with $T_{max}$ equaling the max number of epochs. The Threshold-dependent batch normalization (tdBN) is used in the network to overcome the gradient vanishing or explosion for deep SNNs as suggested by \cite{zheng2021going}. The image data is first normalized by the means and variances of the three channels and then fed into SNNs to trigger spikes. All the experiments are conducted on the PyTorch platform on NVIDIA GeForce RTX 3090. Note that we adapt and modify the implementation of gradient-based attacks of the torchattacks Python package \cite{kim2020torchattacks} as we need to perform successful attacks on SNNs. LIF neurons are used with a threshold of $v_{th}= 1.0$ and a decay factor of $\tau=1$, consistent with previous work~\cite{ding2022snn}. The SG function is configured with $\gamma$=1.

\subsection{Adversarial attack settings}
We mainly explore two types of attack scenarios: in white-box (WB) attacks, the attacker has complete access to the target model, including its architecture and parameters, while in black-box (BB) attacks, the attacker has no direct knowledge of the model and typically constructs a substitute model to approximate the behavior of the target model.
In this work, four classic adversarial attack algorithms are used to evaluate the vulnerability of SNN: Fast Gradient Sign Method (FGSM)~\cite{goodfellow2014explaining}, Projected Gradient Descent (PGD)~\cite{madry2017towards}, Basic Iterative Method (BIM)~\cite{kurakin2018adversarial} and CW~\cite{carlini2017towards} attacks. Given a classification model $f$ with dataset $(x, y_{\text{true}})$, where $x$ is the clean image and $y_{\text{true}}$ is the corresponding correct label. The formulations of the attacks we used in this study are described as follows:

\textbf{FGSM.} FGSM aims to perturb the original data $x$ along the sign direction of the gradient on loss function with one step to increase the perturbed linear output, thus fool the network. It can be formalized as follows:
\begin{equation}
    \hat{x} = x + \epsilon \cdot \text{sign}(\nabla_x \mathcal{L}(f(x), y_{\text{true}})),
\end{equation}
where $\text{sign}(\cdot)$ is an odd mathematical function that extracts the sign of a real number.

\textbf{PGD.} PGD attack is the iterative variant of FGSM. It first starts from a random perturbation in the $L_p$-norm constraint around the original sample $x$, then takes a gradient iteration step in the sign direction to achieve the greatest loss output. It can be formalized as follows:
\begin{equation}
    \hat{x}^0 = x + \mathcal{U}(-\epsilon, +\epsilon),
\end{equation}
\begin{equation}
    \hat{x}^{k+1} = \text{Clip}_{x, \epsilon} \big\{\hat{x}^k + \alpha \cdot \text{sign}(\nabla_x \mathcal{L}(f(\hat{x}^k), y_{\text{true}})) \big\},
\end{equation}
where $k$ is the iterative step, $\alpha$ is the step size for each attack iteration, $\epsilon$ controls the perturbation level. $\mathcal{U}(\cdot)$ is a uniform function, $\text{Clip}_{x, \epsilon}\{x\}$ is the function which performs per-pixel clipping of the image $\hat{x}$, so the result will be in $L_\infty$-norm $\epsilon$-neighborhood of the original image $x$.

\textbf{BIM.} Both BIM and PGD attacks are iterative attacks. Different from PGD attacks, BIM updates the adversarial samples starting from the original image.

\textbf{CW.} CW attack is different from previous gradient-based attack methods. It is based on model optimization to generate adversarial samples. Its optimization function is as follows:
\begin{equation}
    \text{minimize} \frac{1}{2} \lVert \tanh(W) + 1 - x \rVert^2_2 + c \cdot f(\tanh(W) + 1),
\end{equation}
where $f$ is defined as
\begin{equation}
    f(\hat{x}) = \max(\max\{Z(\hat{x})_i : i \neq j\} - Z(\hat{x})_j, 0),
\end{equation}
where $c$ is a parameter to control the perturbation, $Z(\cdot)$ represents the logits output on label $y_i$.

Unlike in ANNs, the input data $x$ is first encoded into sequences $\boldsymbol{O}^0$ over time windows $T$ in SNNs. The most common encoding method involves duplicating $x$ across the temporal dimension~\cite{kim2022rate}, enabling the gradient w.r.t $x$ to be derived from gradient w.r.t $\boldsymbol{O}^0$ as follows:
\begin{equation}
    \nabla_x \mathcal{L}(f(x), y_{\text{true}})=\frac{1}{T}\sum_{t}\frac{\partial\mathcal{L}}{\partial\boldsymbol{O}^0\left(t\right)}.
\end{equation}

\textbf{Backward Pass Through Time.} The most commonly used differentiable approximation is the Backward Pass Through Time (BPTT) with surrogate gradients. In this method, the non-differentiable neuron fire function is replaced by a differentiable function ($h\left(\boldsymbol{\overline{U}}^{l}(t)\right)$) on the backward pass. The backward pass can be described as:
\begin{align}
    \frac{\partial\mathcal{L}}{\partial\boldsymbol{U}^l(t)}
     & =\frac{\partial\mathcal{L}}{\partial\boldsymbol{O}^{l}(t)}\frac{\partial\boldsymbol{O}^{l}(t)}{\partial\boldsymbol{U}^{l}(t)}
    +\frac{\partial\mathcal{L}}{\partial\boldsymbol{U}^l(t+1)}\frac{\partial\boldsymbol{U}^l(t+1)}{\partial\boldsymbol{U}^l(t)}.     \\
     & \approx \frac{\partial\mathcal{L}}{\partial\boldsymbol{O}^{l}(t)} h\left(\boldsymbol{\overline{U}}^{l}(t)\right)
    +\frac{\partial\mathcal{L}}{\partial\boldsymbol{U}^l(t+1)}\frac{\partial\boldsymbol{U}^l(t+1)}{\partial\boldsymbol{U}^l(t)}.
\end{align}

\textbf{Backward Pass Through Rate.} Another differentiable approximation is the Backward Pass Through Rate (BPTR). In this method, the backward pass takes the derivative directly from the average firing rate of the spiking neurons between layers. We consider neurons at each time-step equivalent, and the derivative is determined by the average firing rate.
\begin{align}
    \frac{\partial L}{\partial O^l(t)} = \frac{\partial L}{\partial O^{l+1}(t)} \frac{\partial \frac{1}{T} \sum_{i=0}^{T} O^{l+1}(i)}{\partial \frac{1}{T} \sum_{i=0}^{T} O^l(i)}.
\end{align}
Since the relationship of the firing rates in adjacent layers for non-leaky IF neuron is nearly linear, the gradients of $\frac{\partial \frac{1}{T} \sum_{i=0}^{T} O^{l+1}(i)}{\partial \frac{1}{T} \sum_{i=0}^{T} O^l(i)}$ can be approximated using the straight-through estimator, here we use the constant $\frac{1}{T}$ to approximate the value of $\frac{\partial \frac{1}{T} \sum_{i=0}^{T} O^{l+1}(i)}{\partial \frac{1}{T} \sum_{i=0}^{T} O^l(i)}$ The backward pass of the complete neuronal dynamic is approximated by one single function, and the gradient will not accumulate through time-steps. As the forward pass still follows the rule of the spiking neurons, the obtained gradients will be more accurate than the conversion-based attack.

\section{More experimental results}
\subsection{White-box attack with various timestamps $T$}
In Table~\ref{tb:wb} and Table~\ref{tb:wb_at}, we compare the classification accuracy of MPD-SGR with the vanilla method (REG/RAT~\cite{ding2022snn}) under various white-box attacks. We can see that in different datasets (CIFAR-10 and CIFAR-100), networks (VGG11 and WRN16) and time steps ($T$ = 2, 4, 8), MPD-SGR can effectively and stably improve the robustness of the vanilla model to different attacks. This further verify the effectiveness of our method.

\subsection{Performance under different $\epsilon$ and iterative step $k$}
We evaluate model accuracy degradation under increasing PGD attack intensity $\epsilon$ and iterative step $k$. The results in Figure~\ref{fg:wb-plus} show that MPD-SGR achieves a slower accuracy decline than the baseline under stronger PGD attacks for both VGG and WRN architectures, regardless of AT. This demonstrates that our method exhibits strong tolerance to more intense adversarial perturbations.

\section{Analysis of Gradient Obfuscation}
The main concern of the gradient obfuscation lies in the inaccurate of updating gradients. In particular, the performance of the two differentiable approximations was checked against the five tests that can identify gradient obfuscation as done in \cite{kundu2021hire}. Our analysis is mainly based on the quantification results in Tab. 1 and Fig. 3 in the main text.

As shown in Tab. 1 in the main text, for all the trials in the main text, the performance of single-step FGSM is worse than its iterative counterpart PGD, which certifies the success of BPTT and BPTR approximation in terms of Test(1) in Table \ref{tb:checklist}. To verify Test(2), we conduct black-box attacks on the proposed models and the vanilla ones in Fig. 3 in the main text. The black-box perturbation performs weaker than white-box perturbation, so Test(2) is satisfied. To verify Test(3)(4) we analyze VGG-11 on CIFAR-10 with increasing attack bound. In Figure.\ref{fg:unbound_attack}, the classification accuracy decreases as we increase $\epsilon$ and finally reaches an accuracy of random guessing. As suggested in \cite{kundu2021hire}, Test(5) "can fail only if gradient-based attacks cannot provide adversarial examples for the model to misclassify". To sum up, we found no gradient obfuscation for the BPTT and BPTR approximation, which are suitable for adversarial training and testing.

% ----------- Supplementary Content Ends Here -----------

% References and End of Paper
% These lines must be placed at the end of your paper
\bibliography{19424}